\documentclass[a4paper, 12pt, ]{article}
\usepackage[english]{babel}               % [french, frenchb, english, ]
\usepackage{amsmath,amsfonts,amssymb,amsthm}
\usepackage{lmodern}        % Latin Modern family of fonts. Very much like Computer Modern, but with many more glyphs
\usepackage[T1]{fontenc}    % fontenc is oriented to output, that is, what fonts to use for printing characters.
\usepackage{libertine} \usepackage[libertine]{newtxmath}
\usepackage[utf8]{inputenc} % inputenc allows the user to input accented characters directly from the keyboard;
\usepackage{bm}             % load after all math to give access to bold math
\usepackage{geometry}     % very customizable margins. Under some (rare) circumstances, should be loaded after hyperref
\usepackage[final, babel]{microtype} % many good lay-out/justification effects, see: texblog.net/latex-archive/layout/pdflatex-microtype/
\usepackage{setspace}  % set line spacing
\usepackage{tikz,pgfplots}
\usetikzlibrary{arrows,patterns,plotmarks,decorations.pathreplacing,calc}
\tikzstyle{every picture} += [>=stealth]
\pgfplotsset{compat=newest}
\usepackage{float}                      % Improved interface for floating objects ; add [H] option
\usepackage[margin=10pt,font=small,labelfont=bf]{caption}
\usepackage{subcaption}
\usepackage{multirow}
\usepackage{booktabs}
\usepackage{pgfplotstable}
\usepackage{afterpage}
\usepackage{enumerate}
\usepackage{algorithm}
\usepackage{algorithmicx}
\usepackage{algpseudocode}
\usepackage[normalem]{ulem}
\usepackage{sectsty}   % change section title style
\usepackage{xspace}    % generate space after \ commands
\usepackage{siunitx}
\usepackage{natbib}
\usepackage{xfrac}     % inline fraction
\usepackage[]{authblk}
\usepackage{url}
\usepackage[hypertexnames=false,hyperfootnotes=false]{hyperref}
\newcommand{\field}[1]{\ensuremath{\mathbb{#1}}}
\newcommand{\R}{\ensuremath{\field{R}}} % real numbers
\newcommand{\I}[1]{\ensuremath{\mathbb{I}_{\left\{#1\right\}}}} % indicator function
\newcommand{\PR}{\ensuremath{\mathsf{P}}} % probability
\newcommand{\E}{\ensuremath{\mathsf{E}}} % expectation

\DeclareMathOperator*{\argmax}{\mathrm{argmax}}

\newtheoremstyle{thm-sf}{}{}{\itshape}{}{\sffamily\bfseries}{.}{ }{}
\theoremstyle{thm-sf}
\newtheorem{theorem}{Theorem}
\newtheorem{proposition}{Proposition}

\newtheorem{lemma}{Lemma}
\theoremstyle{definition}
\newtheorem{assumption}{Assumption}

\newtheorem{definition}{Definition}
\theoremstyle{remark}
\newtheorem{remark}{Remark}

\allsectionsfont{\sffamily}
\makeatletter
\def\@seccntformat#1{\csname the#1\endcsname.\quad}
\makeatother
\floatname{algorithm}{\normalfont\sffamily\bfseries Algorithm}
\floatstyle{ruled}
\setcitestyle{authoryear,round,semicolon,aysep={,},yysep={,},notesep={, }}
\title{Nonparametric Pricing Analytics with Customer Covariates}
\author[1]{Ningyuan Chen\thanks{ningyuan.chen@utoronto.ca}}
\author[2]{Guillermo Gallego\thanks{ggallego@ust.hk}}
\affil[1]{Rotman School of Management, University of Toronto}
\affil[2]{Department of Industrial Engineering \& Decision Analytics\authorcr The Hong Kong University of Science and Technology}
\date{}
\begin{document}
\maketitle
\begin{abstract}
Personalized pricing analytics is becoming an essential tool in retailing.
Upon observing the personalized information of each arriving customer, the firm needs to set a price accordingly based on the covariates such as income, education background, past purchasing history to extract more revenue.
For new entrants of the business, the lack of historical data may severely limit the power and profitability of personalized pricing.
We propose a nonparametric pricing policy to simultaneously learn the preference of customers based on the covariates and maximize the expected revenue over a finite horizon. The policy does not depend on any prior assumptions on how the personalized information affects consumers' preferences (such as linear models). It is adaptively splits the covariate space into smaller bins (hyper-rectangles) and clusters customers based on their covariates and preferences, offering similar prices for customers who belong to the same cluster trading off granularity and accuracy.
We show that the algorithm achieves a regret of order $O(\log(T)^2 T^{(2+d)/(4+d)})$, where $T$ is the length of the horizon and $d$ is the dimension of the covariate. It improves the current regret in the literature \citep{slivkins2014contextual}, under mild technical conditions in the pricing context (smoothness and local concavity).
    We also prove that no policy can achieve a regret less than $O(T^{(2+d)/(4+d)})$ for a particular instance and thus demonstrate the near optimality of the proposed policy.
\end{abstract}

%Modern decision analytics frequently involves the optimization of an objective over a finite horizon where the functional form of the objective is unknown. The decision analyst observes covariates and tries to learn and optimize the objective by experimenting with the decision variables. We present a nonparametric learning and optimization policy with covariates.

\section{Introduction}
Personalized pricing refers to the practice that a firm charges customers different prices for the \emph{same} product, depending on customers' information such as education backgrounds and zip codes.
It is increasingly popular in online retailing, as sellers can acquire/infer the personalized information from customers' account profiles or browsing histories (cookies).
The demand (purchasing probability) of each customer depends not only on the price, but also on the personalized information.
The firm observes the information of each arriving customer and sets a personalized price accordingly.
We are interested in finding pricing policies of the firm that maximize the long-run revenue.

Personalized pricing presents several challenges to the firm.
First, for new entrants to the online business, the demand function and how it depends on the personalized information is generally unknown.
Thus, the optimal pricing cannot be obtained by directly solving an optimization problem.
The firm may experiment with different prices to \emph{learn} the personalized demand function, and then sets optimal prices according to the estimation.
There is usually a finite horizon that forces a trade-off between gathering more information (learning/exploration) and making sound decisions (earning/exploitation).
This problem, sometimes referred to as the learning/earning dilemma, has attracted the attention of many researchers.

A second challenge is the presence of personalized information, or \emph{covariates}.
On one hand, the covariate of each customer provides extra information for the firm to predict the personalized demand more accurately.
On the other hand, the demand is peculiar to each instance and changes over time, which adds significant complexity to the learning problem described above.
In particular, learning a market-wise demand function is not sufficient, and the firm
has to learn the personalized demand by experimenting with prices for customers of similar profiles.

A third challenge is the selection of a predictive model to estimate the demand.
Consider a toy example: the demand function only depends on the address of each customer and nothing else.
The firm may postulate a linear model
\begin{equation*}
    \text{demand}=a-b\times \text{price}+\bm c\cdot(\text{latitude, longitude}),
\end{equation*}
and uses historical sales data to learn the parameters $a$, $b$ and $\bm c$ and maximize revenue according to the estimation.
However, whether the model is specified correctly plays an important role in the performance of the pricing policy.
In the particular example, linearity in the location (latitude and longitude) implies that the customers along the straight line that is orthogonal to $\bm c$ have the same demand.
This hardly reflects the reality, as customers from the same neighborhood tend to have similar shopping patterns and neighborhoods are usually clustered geographically.
By postulating a parametric (linear) model, the firm faces the risk of misspecification and not learning what is supposed to be learned.
A \emph{nonparametric} model is more appropriate in this setting.

In this paper we study the pricing policy of a firm which tries to maximize the unknown expected revenue $f(\bm x,p)\triangleq pd(\bm x,p)$, where $p$ is the price, $d(\bm x,p)\in [0,1]$ is the personalized demand function for a customer of covariate $\bm x$.
We assume both quantities are normalized so $\bm x\in[0,1)^d$ and $p\in[0,1]$.
In period $t$, the firm observes an arriving customer with a random covariate $\bm X_t$, and sets a personalized price $p_t$.
The earned revenue is $p_t$ multiplied by a Bernoulli random variable with success rate $d(\bm X_t,p_t)$, representing the event of a purchase.
The expected revenue is thus $f(\bm X_t,p_t)$.

We propose a nonparametric learning policy for the firm. That is, the policy does not depend on specific forms of $f(\bm x,p)$ and only assumes general structures such as continuity and smoothness.
The policy achieves near-optimal performance compared to a clairvoyant who knows $f(\bm x,p)$ and sets $p^\ast(\bm x)=\argmax_{p} f(\bm x,p)$ for a customer of covariate $\bm x$.
More precisely, the expected difference in total revenues between the proposed policy and the clairvoyant policy, which is referred to as the \emph{regret} in the literature, grows at $O(\log(T)^2T^{(2+d)/(4+d)})$ as $T\to \infty$.
The rate is sublinear in $T$, implying that when the length of horizon tends to infinity, the average regret incurred per period becomes negligible.
Moreover, we prove that no pricing policies can achieve a lower regret than $O(T^{(2+d)/(4+d)})$ for a reasonable class of unknown objective functions $f$.
Therefore, we successfully work out the learning/earning dilemma with covariates.

%
%Since the policy is nonparametric and does not depend on a specific form of $f(\bm x, p)$ (we do impose several assumptions on the function; they are much less restrictive than a particular parametric family), it also addresses the third challenge to a great extent.

The main contribution of the paper is the design of a near-optimal nonparametric learning policy for the personalized pricing problem.
Nonparametric learning policies are introduced in more general settings by \citet{rigollet2010nonparametric,perchet2013,slivkins2014contextual}.
The formulation in this paper is originally introduced in \citet{slivkins2014contextual}, and our policy builds upon the idea of adaptive binning proposed in \citet{perchet2013}\footnote{\citet{perchet2013} study discrete decision variables (multi-armed bandit). }.
Motivated by the application of personalized pricing, we assume that the expected revenue $f(\bm x,p)$ is smooth and locally concave in the charged price $p$.
This deviates from the Lipschitz continuous condition in \citet{slivkins2014contextual} and can be viewed as a special ``margin condition'' in \citet{rigollet2010nonparametric,perchet2013}.
By utilizing this condition, we are able to show that our policy is near-optimal and achieves improved regret over merely continuous objective functions ($T^{(2+d)/(4+d)}$ versus $T^{(2+d)/(3+d)}$ in \citet{slivkins2014contextual}).

\subsection{Literature Review}\label{sec:literature}
This paper is motivated by the recent literature that analyzes a firm's pricing problem when the demand function is unknown \citep[e.g.][]{besbes2009dynamic,araman2009dynamic,farias2010dynamic,broder2012dynamic,denboer2014simul,keskin2014dynamic,cheung2017dynamic}.
\citet{den2015dynamic} provides a comprehensive survey for this area.
Since the firm does not know the optimal price, it has to experiment different (suboptimal) prices and update its belief about the underlying demand function.
Therefore, the firm has to balance the exploration/exploitation trade-off, which is usually referred to as the learning-and-earning problem in this line of literature.
Our paper considers the pricing problem with personalized information and it does not consider the finite-inventory setting as in some of the papers mentioned above.

More recently, several papers investigate the pricing problem with unknown demand in the presence of covariates \citep{nambiar2016dynamic,qiang2016dynamic,javanmard2016dynamic,cohen2016feature,ban2017personalized}.
The existing literature has adopted a parametric approach: for example, the actual demand can be expressed in a linear form $\alpha^T \bm x+\beta^T\bm xp+\epsilon$ \citep{qiang2016dynamic,ban2017personalized}, where $\bm x$ is the feature vector of a customer, $\alpha$ and $\beta$ are vectorized coefficients, and $\epsilon$ is the random noise.
Because of the parametric form, a key ingredient in the design of the algorithms in this line of literature is to plug in an estimator for the unknown parameters ($\alpha$ and $\beta$) in addition to some form of forced exploration.
%The estimated value of the parameters allows the firm to compute an estimated optimal price for any given covariate.
In contrast, we focus on a setting where the demand function cannot be parametrized.
Thus, the firm cannot count on accurately estimating the function globally by estimating a few parameters.
Instead, a localized optimal decision has to be made based on past covariates generated in the neighborhood.
It highlights the different philosophies when designing algorithms for parametric/nonparametric learning problems with covariates.
As a result, the best achievable regret deteriorates from $O(\sqrt{T})$ or $O(\log T)$ (parametric) to $O(T^{(2+d)/(4+d)})$ (nonparametric).

The dependence of the optimal rate of regret on the problem dimension $d$ has been observed before.
For example, \citet{cohen2016feature} find a multi-dimensional binary search algorithm for feature-based dynamic pricing, which has regret $O(d^2\log(T/d))$;
\citet{javanmard2016dynamic} propose a policy for a similar problem that achieves regret $O(s\log d\log T)$, where $s$ represents the sparsity of the $d$ features;
in \citet{ban2017personalized}, the near-optimal policy achieves regret $O(s\sqrt{T})$.
In their parametric frameworks, the dependence of the regret on $d$ is rather mild---it does not appear on the exponent of $T$;
\citet{javanmard2016dynamic,keskin2014dynamic} also provide methods to deal with the sparse strucutre.
In contrast, in our nonparametric formulation, the optimal rate of regret $O(T^{(2+d)/(4+d)})$ increases dramatically in $d$, making the problem significantly harder to learn in high dimensions.
This is similar to the nonparametric formulation in the network revenue management problem \citep{besbes2012blind},
in which the dimension of the decision space is $d$ and the optimal rate of regret is $O(T^{(2+d)/(3+d)})$\footnote{It is shown in \citet{chen2018primal} that if the number of inventory constraints $\ll d$, then learning the dual variables may effectively reduce the problem dimension.}.
From the literature, it seems that the dimension significantly complicates the learning problem in a nonparametric formulation.

This paper is also related to the vast literature studying multi-armed bandit problems. See \citet{cesa2006prediction,bubeck2012regret} for a comprehensive survey.
The classic multi-armed bandit problem involves finite arms, and the algorithms \citep{kuleshov2014algorithms,agrawal2012analysis} cannot be applied directly to our setting.
Recently, there is a stream of literature studying the so-called continuum-armed bandit problems \citep{agrawal1995continuum,kleinberg2005nearly,Auer2007,kleinberg2008multi,bubeck2011x}, in which there are infinite number of arms (decisions).
Although there is no contextual information in those papers, \citet{kleinberg2008multi,bubeck2011x} have developed algorithms based on a similar idea to decision trees, because the potential arms form a high-dimensional space.

For multi-armed bandit problems with contextual information, parametric and regression-based algorithms have been proposed in, for example, \citet{goldenshluger2013linear,bastani2015online}.
Our paper is related to the literature studying contextual multi-armed bandit problems in a nonparametric framework \citet{yang2002randomized,langford2008epoch,rigollet2010nonparametric,perchet2013,slivkins2014contextual,elmachtoub2017practical}.
The analysis builds upon the idea of adaptive binning in \citet{perchet2013}.
Our algorithm is designed for continuous decisions.
In fact, applying the algorithm in \citet{perchet2013} designed for discrete decisions to our problem with simple discretization
may result in worse-than-optimal regret.
In terms of the formulation and the rate of regret, this paper is closely related to \citet{slivkins2014contextual}.
\citet{slivkins2014contextual} investigates a more general model, in which the decision $p$ can be a vector,
and assumes that $f(\bm x,p)$ is Lipschitz continuous.
The optimal rate of regret in this setting has been shown to be $T^{(1+d_x+d_p)/(2+d_x+d_p)}$ in previous works,
where $d_x$ and $d_p$ are the dimensions of $\bm x$ and $p$, respectively.
\citet{slivkins2014contextual} introduces an adaptive zooming algorithm, and uses the covering dimension of the space $(\bm x,p)$ in the analysis.
The extension accommodates more general spaces of $(\bm x,p)$ than the Euclidean space.
For Euclidean spaces, the algorithm recovers the optimal rate of regret $T^{(1+d_x+d_p)/(2+d_x+d_p)}$.
In our setting, letting $d_x=d$ and $d_p=1$ leads to the rate $T^{(2+d)/(3+d)}$.\footnote{Applying our assumptions and following Equation (8) of \citet{slivkins2014contextual} give $d_x+d_p=d+1/2$, and the regret improves to $T^{(1.5+d)/(2.5+d)}$.}
Motivated by personalized pricing, we impose additional assumptions on $f(\bm x,p)$ (smoothness and local concavity, see Assumption~\ref{asp:maximizer}), and improve the optimal rate to $T^{(2+d)/(4+d)}$ as a result.
The additional assumption may act as a special margin condition \citep{tsybakov2004optimal,rigollet2010nonparametric,perchet2013} and affects the optimal rate.
It is unclear whether the algorithm in \citet{slivkins2014contextual} could be adapted to accommodate the additional assumptions and achieve the improved rate.
The design of our algorithm is based on adaptively partitioning the covariate space into rectangular bins, rather than overlapping balls as in \citet{slivkins2014contextual}.

%Our algorithm differs from those two papers in the following aspects: the algorithm design, the assumptions, and the achieved optimal rate of regret.
%For clarity, we defer the comparison and discussion to Section~\ref{sec:discussion} after the analysis of our algorithm.
%
%and nonparametric algorithms have been proposed in \citet{rigollet2010nonparametric,perchet2013}.
%Among them, \citet{perchet2013} is the closest to our paper in terms of methodology.
%They introduce the idea of adaptive binning to learning problems with covariates and combine it with sequential arm eliminations.
%However, the fundamental difference between discrete arms in their paper and continuous decisions in our paper require different sets of assumptions and innovative treatments for the exploration/exploitation trade-off.

\section{Problem Formulation}
Suppose the personalized demand function (purchasing probability) is $d(\bm x,p)\in[0,1]$, where $\bm x\in[0,1)^d$ is the observed feature vector, or covariate, summarizing the customer's personalized information, and $p$ is the price set by the firm.
Since the purchasing event is a Bernoulli random variable with success rate $d(\bm x,p)$ when the firm sets price $p$ for a customer of covariate $\bm x$ (henceforth abbreviated to customer $\bm x$), the expected revenue is thus $f(\bm x,p)\triangleq p d(\bm x,p)$.
If the firm knew $d(\bm x,p)$, or equivalently, $f(\bm x,p)$, then it would set $p^{\ast}(\bm x)\triangleq \argmax_{p\in[0,1]} f(\bm x,p)$.
Denote the optimal expected revenue from customer $\bm x$ by $f^{\ast}(\bm x)\triangleq \max_{p\in[0,1]}f(\bm x,p)$.
We will be primarily dealing with the expected revenue $f(\bm x,p)$ instead of the personalized demand $d(\bm x,p)$.

Initially, neither $d(\bm x,p)$ nor $f(\bm x,p)$ is known to the firm.
In period $t\in\left\{1,2,\dots,T\right\}$, a customer arrives with covariate $\bm X_t$.
Upon observing $\bm X_t$, the firm sets a price $p_t$.
The revenue earned in period $t$ is denoted by $Z_t$ where $Z_t/p_t$ is a Bernoulli random variable with mean $d(\bm X_t,p_t)$, independent of everything else.
The objective of the firm is to design a pricing policy to maximize the total revenue over the horizon $\sum_{t=1}^T \E[Z_t]=\sum_{t=1}^T \E[f(\bm X_t,p_t)]$.
Note that $p_t$ itself is likely to be random even though the firm is not adopting a randomized policy.
This is because the pricing decision made in period $t$ may depend on the observed customers, set prices, and earned revenues in the previous periods.
That is, $p_t= \pi_t(\bm X_1,p_1,Z_1,\dots,\bm X_{t-1},p_{t-1},Z_{t-1},\bm X_{t})$.
Formally, we refer to $\pi$ as the \emph{pricing policy} that determines how $p_t$ depends on the past information.
We also denote $\mathcal F_{t}\triangleq\sigma(\bm X_1,p_1,Z_1,\dots,\bm X_{t},p_t,Z_{t})$.

%The main application of our model is personalized dynamic pricing. In this case $f(\bm x,p)=(p-c)d(\bm x,p)$ is the revenue/profit function and $d(\bm x,p)$ is the expected demand at $(\bm x,p)$, and $c$ is the unit cost.  Our method, however, can also be applied to the supply side where the price $p$ is fixed and we are trying to maximize the expected profit $f(\bm x,q) = p\E[\min \{D(\bm x,p),q\}] - cq$ over the order quantity $q$.

%
%
%
%
%
%
%Besides demand analytics, the formulation can also be applied to the supply side.
%Suppose that an item can be purchased at unit cost $c$ and sold at fixed unit price $p > c$; the random demand associated with the observed market indicator $\bm x$ is $D(\bm x)$ with unknown distribution. Then the retailer needs to determine the order quantity $q$ to maximize the newsvendor profit $f(\bm x,q) = p\E[\min \{D(\bm x),q\}] - cq$.

%This problem is similar to the multi-armed bandit problem with covariates, with the difference that the
%decision variable is a continuous variable rather than a set of arms in this problem.
%In the OR community, some recent papers also investigate this learning and earning problem with covariates \citep{ban2017personalized,javanmard2016dynamic,qiang2016dynamic,nambiar2016dynamic}.
%The major contribution of this paper is that we explore the nonparametric functions, which contains the parametric (linear) setting studied in the literature, and design an algorithm that can almost achieve the lower bound.
%The asymptotic regret shown in this paper is significantly different from the parametric setting.

\subsection{Regret}\label{sec:regret}
To measure the performance of a pricing policy, it is standard in the literature to benchmark it against the so-called \emph{clairvoyant} policy and study the \emph{regret}.
Suppose $f(\bm x,p)$ is known to a clairvoyant firm.
The optimal pricing policy is rather straightforward for a clairvoyant:
having observed customer $\bm X_t$, set $p^{\ast}(\bm X_t)$ in period $t$ and earn a random revenue with mean $f^{\ast}(\bm X_t)$.

For the firm, the expected revenue in period $t$ cannot exceed that of the clairvoyant: $f(\bm X_t,p_t)\le f^{\ast}(\bm X_t)$.
Thus, we define the regret of a pricing policy $\pi$ to be the revenue gap
\begin{equation*}
    R_{\pi}(T) = \sum_{t=1}^T \E\left[(f^{\ast}(\bm X_t)-f(\bm X_t,p_t))\right].
\end{equation*}
In period $t$, the expectation is taken with respect to the distribution of $\bm X_t$ as well as $p_t$,
which itself depends on $\mathcal F_{t-1}$ and $\bm X_t$.
Our goal is to design a policy $\pi$ that achieves small $R_{\pi}(T)$ when $T\to\infty$.

However, because $R_{\pi}(T)$ also depends on the unknown function $f$,
we require the designed policy to perform well
for a family $\mathcal C$ of functions, i.e., we seek for optimal policies in terms of the minimax regret
\begin{equation*}
    \inf_{\pi}\sup_{f\in \mathcal C} R_{\pi}(T).
\end{equation*}
Although it is usually impossible to find the exact policy that achieves the minimax regret,
we focus on proposing a policy whose regret is at least comparable to (of the same order as) the minimax regret asymptotically when $T\to \infty$.

In this paper, we study functions $f$ that cannot be parametrized.
It implies that the family $\mathcal C$ is much larger than parametric families: it includes all the functions that satisfy some mild assumptions presented in the next section.
In other words, the worst-case scenario can potentially be much worse than a parametric family.
As a result, the achievable minimax regret is also higher.

%Evaluating policies by $\sup_{f\in \mathcal C} R_{\pi}(T)$ is standard in the literature.
%The rationale for such measure can be illustrated by the following simple example.
%A naive policy $\pi_t\equiv 0$ may perform no worse than the clairvoyant policy for $f(\bm x,p)\equiv -p^2$;
%but is drastically outperformed by other well designed policies for a general reward function.
%Therefore, it is reasonable to focus on the worst-case regret in terms of possible reward functions.
%As a result, the class of functions $\mathcal C$ is particularly important in the regret analysis.
%In general, a larger class implies a worse minimax regret and vice versa.

\subsection{Assumptions}\label{sec:assumption}
%It is not surprising that the class of functions $f(\bm x,p)$ has to be restricted.
%In an extreme counterexample, if the function $f(\bm x,p)$ is highly discontinuous in $p$, then there is no hope that the decision analyst can effectively learn the form of $f$.
In this section, we formally provide a set of assumptions that $f\in\mathcal C$ and the stochastic process has to satisfy and their justifications.

\begin{assumption}\label{asp:sub-gaussian}
    The covariates $\bm X_t$ are i.i.d. for $t=1,\ldots,T$. Given $X_t$ and $p_t$, the revenue $Z_t$ is independent of everything else.
\end{assumption}
Both i.i.d. covariates and independent noise structure are standard in the literature.

\begin{assumption}\label{asp:lipschitz}
    The functions $f(\cdot,p)$ and $f(\bm x,\cdot)$ are Lipschitz continuous given $p$ and $\bm x$, i.e., there exists $M_1>0$ such that
        $|f(\bm x_1,p)-f(\bm x_2,p)|\le M_1 \|\bm x_1-\bm x_2\|_2$ and $|f(\bm x,p_1)-f(\bm x,p_2)|\le M_1 |p_1-p_2|$ for all $\bm x_i$ and $p_i$ ($i=1,2$) in the domain.
\end{assumption}
This assumption is equivalent to $|f(\bm x_1,p_1)-f(\bm x_2,p_2)|\le M_1 (\|\bm x_1-\bm x_2\|_2+|p_1-p_2|)$.
Lipschitz continuity is a common assumption in the learning literature.
In personalized pricing, it implies that the expected revenues are close if the firm charges similar prices for two customers with similar covariates.
If this assumption fails, then the historical sales data of a certain type of customer is not informative for a new customer with almost identical background and learning is virtually impossible.

To introduce the next assumption, consider any hyper-rectangle $B\subset [0,1)^d$, including a singleton $B=\{\bm x\}$.
Define $f_B(p)\triangleq \E\left[f(\bm X,p)|\bm X\in B\right]$ for $p\in[0,1]$.
Clearly $f_B(p)$ is the expected revenue when charging $p$ for a customer that is sampled from a subset $B$.
\begin{assumption}\label{asp:maximizer}
    We assume that for any $B$,
    \begin{enumerate}
        \item The function $f_B(p)$ has a unique maximizer $p^{\ast}(B)\in[0,1]$. Moreover, there exist uniform constants $M_2, M_3>0$ such that for all $p\in [0,1]$, $M_2(p^{\ast}(B)-p)^2\le f_B(p^{\ast}(B))-f_B(p)\le M_3(p^{\ast}(B)-p)^2$.
        \item The maximizer $p^{\ast}(B)$ is inside the interval $[\inf\{p^{\ast}(\bm x):\bm x\in B\},\sup\{p^{\ast}(\bm x):\bm x\in B\}]$.
        \item Let $d_B$ be the diameter of $B$. Then there exists a uniform constant $M_4>0$ such that $\sup\{p^{\ast}(\bm x):\bm x\in B\}-\inf\{p^{\ast}(\bm x):\bm x\in B\}\le M_4 d_B$.
    \end{enumerate}
\end{assumption}
This assumption is quite different from those in the setting without covariates \citep{besbes2009dynamic,wang2014close,lei2014near} or the parametric setting \citep{ban2017personalized,qiang2016dynamic}.
To explain the intuition of $f_B(p)$,
suppose the firm only observes $\I{\bm X\in B}$ but not the exact value of $\bm X$, and thus cannot apply personalized pricing for customers in $B$.
In this case, the learning objective is the expected revenue $f_B(p)$ and
the clairvoyant policy that has the knowledge of $f(\bm x,p)$ is to set $p=p^{\ast}(B)$.
This class of learning problems are important subroutines of the algorithm we propose and Assumption~\ref{asp:maximizer} guarantees that they can be effectively learned.

For part one of Assumption~\ref{asp:maximizer}, we have
\begin{proposition}\label{prop:quadratic}
    If $f_B(p)$ is continuous for $p\in[0,1]$, and twice differentiable in an open interval containing the unique global maximizer $p^{\ast}(B)$ with $f_B''(p^{\ast}(B))<0$, then part one of Assumption~\ref{asp:maximizer} holds.
\end{proposition}
Therefore, part one states that $f_B(p)$ is smooth and locally concave around the maximum.
If $B$ is a singleton, then it can be viewed as a weaker version of the concavity assumption in \citet{wang2014close,lei2014near}, i.e., $0>a>f''(p)>b$ for all $p$ in their no-covariate setting.
As a result, if $f_B(p)$ is the revenue function of linear or exponential demand, then part one is satisfied automatically.

Part two of Assumption~\ref{asp:maximizer} prevents the following scenario:
If the optimal price for the aggregate demand of customers $\bm x\in B$, $p^{\ast}(B)$,
is far from the optimal personalized pricing $p^{\ast}(\bm x)$,
then collecting more information for $f_B(p)$ does not help to improve the pricing decision for any individual customer $\bm x \in B$.
Such obstacle may lead to failure to learn and is thus ruled out by the assumption.
Similar types of assumptions have been imposed in other applications of revenue management.
For example, Proposition 1.16 in \citet{ggbook} provides conditions under which the optimal price of the aggregated market lies in the convex hull formed by the optimal prices of each market segment when discriminatory pricing is allowed.

Part three imposes a continuity condition for the optimal price.
It is equivalent to, for example, some form of continuous differentiability of $f(\bm x, p)$, because $p^{\ast}(\bm x)$ solves the implicit function from the first-order condition $f_p(\bm x,p)=0$.

\begin{remark}
    Assumption~\ref{asp:lipschitz} is a variant of similar assumptions adopted in the literature.
    Assumption~\ref{asp:maximizer}, although appearing nonstandard, is also satisfied by the parametric families studied by previous works. We give a few examples that satisfy Assumption~\ref{asp:maximizer}.
    \begin{itemize}
        \item Dynamic pricing with linear covariate \citep{qiang2016dynamic}: if $f(\bm x,p)=p(\bm \theta^T\bm x-\alpha p)$, then $f_B(p)=p(\bm\theta^T\E[\bm X|\bm X\in B]-\alpha p)$ and $p^{\ast}(B)=\bm\theta^T\E[\bm X|\bm X\in B]/2\alpha$.
        \item Separable function: consider $f(\bm x,p)=\sum_{i=1}^k g_i(\bm x)h_i(p)$. Then $f_B(p)=\sum_{i=1}^k \E[g_i(\bm X)|\bm X\in B]h_i(p)$. If $h_i(p)$ are concave functions and $g_i(\bm x)$ are positive, then we may be able to solve the unique maximizer $p^{\ast}\left(B\right)=\E[g(\bm X)|\bm X\in B]$ for some continuous function $g$.
        \item Localized functions: the covariate only plays a role in a subset $B_0\subset [0,1)^d$.
            See Section~\ref{sec:lower-bound} for a concrete example.
    \end{itemize}
    As we shall see in Sections~\ref{sec:upper} and \ref{sec:lower-bound}, the optimal rate of regret under Assumption~\ref{asp:maximizer} is $T^{(2+d)/(4+d)}$, in contrast to $T^{(2+d)/(3+d)}$ without it (taking $d_Y=1$ in Equation (3) of \citealt{slivkins2014contextual}).
    Technically, we suspect that Assumption~\ref{asp:maximizer} plays a similar rule to the margin condition in the contextual bandit literature \citep{tsybakov2004optimal,goldenshluger2009woodroofe,rigollet2010nonparametric,perchet2013}.
    However, because of the continuous decision variable studied in this paper, the margin condition cannot be translated in a straightforward way.
    It remains a future direction to present the assumption in a general form (the degree of smoothness such as the H\"{o}lder and Sobolev classes) and study how it affects the optimal rate of regret.
\end{remark}
We summarize the information available to the firm.
In the beginning of the horizon, the length of the horizon $T$, the dimension $d$ and the constants $\{M_i\}_{i=1}^4$ are revealed\footnote{In fact, only $M_2$ is needed in the algorithm.}.
In period $t$, the price can also depend on $\mathcal F_{t-1}$ and $\bm X_t$.

\section{The ABE Algorithm}
We next present a set of preliminary concepts related to the \emph{bins} of the covariate space, and then introduce the proposed pricing policy: the Adaptive Binning and Exploration (ABE) algorithm.

\subsection{Preliminary Concepts}
\begin{definition}
    A bin is a hyper-rectangle in the covariate space. More precisely, a bin is of the form
    \begin{equation*}
    B = \left\{\bm x: a_i\le x_i<b_i,\;i=1,\dots,d\right\}
\end{equation*}
for $0\le a_i<b_i\le 1, i = 1,\dots,d$.
\end{definition}
We can \emph{split} a bin $B$ by bisecting it in all the $d$ dimensions to obtain  $2^d$ \emph{child} bins of $B$, all of equal size.
For a bin $B$ with boundaries $a_i$ and $b_i$ for $i=1,\ldots,d$, its children are indexed by $\bm i\in\{0,1\}^d$ and have the form
\begin{equation*}
    B_{\bm i}=\left\{\bm x: a_j\le x_j< \frac{a_j+b_j}{2}\text{ if $\bm i_j=0$, }\frac{a_j+b_j}{2}\le x_j < b_j\text{ if $\bm i_j=1$},\; j=1,\dots,d\right\}.
\end{equation*}
Denote the set of child bins of $B$ by $\bm C(B)$.
Conversely, for any $B'\in \bm C(B)$, we refer to $B$ as the \emph{parent} bin of $B'$, denoted by $P(B')=B$.

Our algorithm starts with a root bin $B_{\emptyset}\triangleq [0,1)^d$, which contains all possible customers, and successively splits the bin as more data is collected.
Therefore, any bin $B$ produced during the process is the \emph{offspring} of $B_{\emptyset}$, i.e., $P^{(k)}(B)=B_{\emptyset}$ for some $k>0$, where $P^{(k)}$ is the $k$th composition of the parent function.
Equivalently, $B_{\emptyset}$ is an \emph{ancestor} of $B$.
%Therefore, one can use a sequence of indices $(\bm i_1,\bm i_2,\dots,\bm i_k)$ to represent a bin.
%As introduced above, the index $\bm i$ encodes the reference to a particular child bin when a parent bin is split.
%Likewise, $(\bm i_1,\bm i_2,\dots,\bm i_k)$ refers to a bin that is obtained by $k$ split operations from $B_{\emptyset}$:
%when $B_{\emptyset}$ is split, we obtain its child $B_{\bm i_1}$;
%when $B_{\bm i_1}$ is split, we obtain its child $B_{\bm i_1\bm i_2}$; and so on.
%In the last operation, when $B_{\bm{i}_1 \ldots \bm i_{k-1}}$ is split, we obtain its child $B_{\bm i_1\ldots\bm i_{k}}$.
For such a bin, we define its \emph{level} to be $k$, denoted by $l(B)=k$. Conventionally, let $l(B_{\emptyset})=0$.

In the algorithm, we keep a dynamic partition $\mathcal P_t$ of the covariate space consisting of the offspring of $B_{\emptyset}$ in each period $t$. The partition is mutually exclusive and collectively exhaustive, so $B_i\cap B_j=\emptyset$ for $B_i,B_j\in \mathcal P_t$, and $\cup_{B_i\in \mathcal P_t}B_i=B_{\emptyset}$.
Initially $\mathcal P_0= \{B_{\emptyset}\}$.
In the algorithm, we gradually refine the partition; that is, each bin in $\mathcal P_{t+1}$ has an ancestor (or itself) in $\mathcal P_{t}$.
%An analogous, and probably more graphical, interpretation is to regard the sequential splitting as a \emph{branching process} and relate it to decision trees in statistical learning.
%Consider $B_{\emptyset}$ as the \emph{root} of a tree, or the initial \emph{leaf} of the tree.
%When a split operation is performed, a leaf is branched into $2^d$ leaves.
%During the branching process, the set of all terminal leaves (those without offspring) form a partition of the covariate space.
%The algorithm involves gradually branching the tree as $t$ increases and more data is collected.

\subsection{Intuition}
The intuition behind the ABE algorithm is to use a partition $\mathcal P_t$ of the covariate space
to aggregate customers in each period.
It tries to find the optimal price for customers in each bin $B\in \mathcal P_t$, i.e., $p^{\ast}(B)$ defined in Section~\ref{sec:assumption}.
As $\mathcal P_t$ is refined dynamically, i.e., $B\in \mathcal P_t$ becomes smaller, such aggregation is almost identical to personalized pricing.

To do that, we keep a set of discrete prices (referred to as the \emph{decision set} hereafter) for each bin in the partition.
The decision set consists of equally spaced grid points of a price interval associated with the bin.
When a customer arrives with covariate $\bm X_t$ inside a bin $B$,
a price is chosen successively in the decision set and charged for the customer.
The realized revenue for this price is recorded.
When a large number of customers are observed in $B$, the average revenue for each price $p$ in the decision set is close to $f_B(p)$, which is defined as $\E[f(\bm X,p)|\bm X\in B]$ in Section~\ref{sec:assumption}.
Therefore, the \emph{empirically-optimal} price in the decision set is close to $p^{\ast}(B)$, with high confidence.

There are two potential pitfalls of this approach.
First, the number of prices has an impact on the performance of the policy.
If there are too many prices in a decision set, then for a given number of customers observed in the bin,
each price is experimented with for a relatively few times.
As a result, the confidence interval for the associated average revenue is wide.
On the other hand, if there are too few prices, then inevitably the decision set has low resolutions.
That is, the optimal price in the set could still be far from the true maximizer $p^{\ast}(B)$ because of the discretization error.
We have to select a proper size of the decision set to balance this trade-off.

Second, even if the optimal price $p^{\ast}(B)$ for the aggregate revenue in the bin is correctly identified,
it may not be a strong indicator for $p^{\ast}(\bm x)$ for a particular customer $\bm x\in B$.
Indeed, $f_B(p)$ averages out all customers $\bm X\in B$, and the optimal price for an individual customer $\bm x$ could be very different.
This obstacle, however, can be overcome as the size of $B$ decreases, as implied by Assumption~\ref{asp:maximizer}.
In particular, part two and three of the assumption guarantee that when $B$ is small,
$p^{\ast}(\bm x)$ is concentrated within a neighborhood of $p^{\ast}(B)$ as long as $\bm x\in B$.
The cost of using a smaller bin, however, is the less frequency of observing a customer inside it.

To remedy the second pitfall, the algorithm adaptively refines the partition and decreases the size of the bins in $\mathcal P_t$ as $t$ increases.
When a bin $B\in\mathcal P_t$ is large, the aggregate optimal price $p^{\ast}(B)$ is not a strong indicator for $p^{\ast}(\bm x)$, $\bm x\in B$.
As a result, we only need a rough estimate and split the bin when a relatively small number of customers are observed in $B$.
When a bin $B\in\mathcal P_t$ is small, the optimal price $p^{\ast}(B)$ provides a strong indicator for $p^{\ast}(\bm x)$, $\bm x\in B$.
Therefore, we gather large sales data from customers $\bm X\in B$ to explore the decision set and estimate $p^{\ast}(B)$ accurately, before it splits.
%The algorithm splits a bin as sufficient covariate data is aggregated in it, until the bin is small enough.

A crucial step in the algorithm is to determine what information to inherit when a bin is split into child bins.
The ABE algorithm records the empirically-optimal price in the decision set of the parent bin.
In the child bins, we use this information and set up their decision sets centered at it.
As explained above, when the parent bin (and thus the child bins) is large, its optimal price does not predict those of the child bins well.
Therefore, the algorithm sets up conservative decision sets for the child bins, i.e., they have wide intervals.
On the other hand, when the parent bin is small, its optimal price provides an accurate indicator for those of the child bins.
Thus, the algorithm constructs decision sets with narrow ranges for the child bins around the empirically-optimal price inherited.

\begin{algorithm}
    \caption{Adaptive Binning and Exploration (ABE)}
    \label{alg:abe}
    \begin{algorithmic}[1]
        \State Input: $T$, $d$
        \State Constants: $M_1$, $M_2$, $M_3$, $M_4$
        \State Parameters: $K$; $\Delta_k$, $n_k$, $N_k$ for $k=0,\ldots,K$
        \State Initialize: partition $\mathcal P\gets \{B_{\emptyset}\}$, $p_l^{B_{\emptyset}}\gets 0$, $p_u^{B_{\emptyset}}\gets 1$, $\delta_{B_{\emptyset}}\gets 1/(N_{0}-1)$, $\bar Y_{B,j},N_{B_{\emptyset},j}\gets 0$ for $j=0,\ldots, N_{0}-1$\label{step:initialize}
        \For{$t=1$ to $T$}
        \State Observe $\bm X_t$\label{step:generate-cov}
        \State $B\gets\{B\in \mathcal P: \bm X_t\in B\}$\label{step:partition}
        \Comment{The bin in the partition that $\bm X_t$ belongs to}
        \State $k\gets l(B)$, $N(B)\gets N(B)+1$\label{step:counter-nb}
        \Comment{Determine the level and update the number of customers observed in $B$}
        \If{$k<K$}\label{step:kK}
        \Comment{If not reaching the maximal level $K$}
        \If{$N(B)< n_k$}\label{step:nb<nk}
        \Comment{If not enough data observed in $B$}
        \State $j\gets N(B)-1\pmod{N_k}$\label{step:nb-mod-nk}
        \Comment{Apply the $j$th price in the decision set}
        \State $p_t\gets p_l^B+j \delta_B$; apply $p_t$ and observe $Z_t$\label{step:policy}
        \State $\bar Y_{B,j}\gets \frac{1}{N_{B,j}+1}(N_{B,j}\bar Y_{B,j}+Z_t)$, $N_{B,j}\gets N_{B,j}+1$\label{step:update-avg}
        \Else\label{step:enough-data}
        \Comment{If sufficient data observed in $B$}
        \State $j^{\ast}\in \argmax_{j\in\{0,1,\ldots,N_k-1\}}\{\bar Y_{B,j}\}$, $p^{\ast}\gets p_l^B+j^{\ast} \delta_B$\label{step:empirical-mean}
        \Comment{Find the empirically-optimal price; if there are multiple, choose any one of them}
%        \Comment{For child bins of $B$, explore around $p^{\ast}$}
        \State $\mathcal P\gets (\mathcal P\setminus B)\cup \bm C(B)$\label{step:partition-update}
        \Comment{Update the partition by removing $B$ and adding its children}
        \For{$B'\in \bm C(B)$}
        \Comment{Initialization for each child bin}
        \State $N(B')\gets 0$\label{step:child-counter}
        \State $p_l^{B'}\gets \max\{0, p^{\ast}-\Delta_{k+1}/2\}$; $p_u^{B'}\gets \min\{1, p^{\ast}+\Delta_{k+1}/2\}$\label{step:child-interval}
        \Comment{The range of the decision set}
        \State $\delta_{B'}\gets (p_u^{B'}-p_l^{B'})/(N_{k+1}-1)$\label{step:def-deltab}
        \Comment{The grid size of the decision set}
        \State $N_{B',j},\bar Y_{B',j}\gets 0$, for $j=0,\ldots, N_{k+1}-1$\label{step:child-reward-init}
        \Comment{Initialize the average revenue and number of customers for each price}
        \EndFor
        \EndIf
        \Else
        \Comment{If reaching the maximal level}
        \State $p_t\gets (p_l^{B}+p_u^{B})/2$\label{step:policy-K}
        \EndIf
        \EndFor
    \end{algorithmic}
\end{algorithm}

\subsection{Description of the Algorithm}
In this section, we elaborate on the detailed steps of the ABE algorithm, shown in Algorithm~\ref{alg:abe}.

The parameters for the algorithm include
\begin{enumerate}
    \item $K$, the maximal level of the bins. When a bin is at level $K$, the algorithm no longer splits it and simply applies the median price of its decision set whenever a customer is observed in it.
    \item $\Delta_k$, the length of the interval that contains the decision set of level-$k$ bins.
    \item $n_k$, the maximal number of customers observed in a level-$k$ bin in the partition. When $n_k$ customers are observed, the bin splits.
    \item $N_k$, the number of prices to explore in the decision set of level-$k$ bins. The decision set of bin $B$ consists of equally spaced grid points of an interval $[p_l^{B},p_u^B]$, to be adaptively specified by the algorithm.
\end{enumerate}
We initialize the partition to include only the root bin $B_{\emptyset}$ in Step~\ref{step:initialize}.
Its decision set spans the whole interval $[0,1]$ with $N_0$ equally spaced grid points.
That is, the $j$th price is $j\delta_{B_{\emptyset}}\triangleq j/(N_0-1)$ for $j=0,\dots,N_{0}-1$.
The initial average revenue and the number of customers that are charged the $j$th price are set to $\bar Y_{B_{\emptyset},j}=N_{B_{\emptyset},j}=0$.

Suppose the partition is $\mathcal P_t$ at $t$ and a customer $\bm X_t$ is observed (Step~\ref{step:generate-cov}).
The algorithm determines the bin $B\in \mathcal P_t$ which the customer belongs to.
The counter $N(B)$ records the number of customers already observed in $B$ up to $t$ when $B$ is in the partition (Step~\ref{step:counter-nb}).
If the level of $B$ is $l(B)=k<K$ (i.e., $B$ is not at the maximal level) and the number of customers observed in $B$ is not sufficient (Step~\ref{step:kK} and Step~\ref{step:nb<nk}),
then the algorithm has assigned a decision set to the bin in previous steps, namely,
$\{p_l^B+j\delta_B\}$ for $j=0,\dots,N_k-1$.
There are $N_k$ prices in the set and they are equally spaced in the interval $[p_l^{B},p_u^B]$.
They are explored successively as new customers are observed in $B$ (explore $p_{l}^B$ for the first customer observed in $B$, $p_l^B+\delta_B$ for the second customer, \dots, $p_l^B+(N_k-1)\delta_B$ for the $N_k$th customer, $p_l^B$ again for the $(N_k+1)$th customer, etc.).
Therefore, the algorithm charges price $p_t=p_l^B+j\delta_B$ where $j=N(B)-1 \pmod{N_k}$ for the $N(B)$th customer observed in $B$ (Step~\ref{step:nb-mod-nk}).
Then, Step~\ref{step:update-avg} updates the average revenue and the number of customers for the $j$th price.

If the level of $B$ is $l(B)=k<K$ and we have observed a sufficient number of customers in $B$ (Step~\ref{step:kK} and Step~\ref{step:enough-data}),
then the algorithm splits $B$ and replaces it by its $2^d$ child bins in the partition (Step~\ref{step:partition-update}).
For each child bin, Step~\ref{step:child-counter} to Step~\ref{step:child-reward-init} initialize
the counter, the interval that encloses the decision set, the grid size of the decision set, and the average revenue/number of customers for each price in the decision set, respectively.
In particular, to construct the decision set of a child bin, the algorithm first computes the empirically-optimal price in the decision set of the parent bin $B$; that is, $j^{\ast}\in \argmax_{j\in\{0,1,\ldots,N_k-1\}}\{\bar Y_{B,j}\}$ in Step~\ref{step:empirical-mean}.
Then, the algorithm creates an interval centered at this empirically-optimal price with width $\Delta_{k+1}$, properly cut off by the boundaries $[0,1]$.
The decision set is then an equally spaced grid of the above interval (Step~\ref{step:child-interval} and Step~\ref{step:def-deltab}).
%Hence all child bins of $B$ have a common decision set.

If the level of $B$ is already $K$, then the algorithm simply charges the median price (Step~\ref{step:policy-K}) repeatedly without further exploration.
For such a bin, its size is sufficiently small and the algorithm has narrowed the range of the decision set $K$ times.
The charged price is close enough to all $p^{\ast}(\bm x)$, $\bm x\in B$, with high probability.

\subsection{Choice of Parameters}\label{sec:choice}
We set $K=\lfloor \frac{\log(T)}{(d+4)\log (2)}\rfloor$, $\Delta_k = 2^{-k}\log(T)$, $N_k=\lceil\log(T)\rceil$, and
\begin{equation*}
        n_k = \max\left\{0, \left\lceil \frac{ 2^{4k+15}}{M^2_2\log^3(T)}(\log(T)+\log(\log(T))-(d+2)k\log(2))\right\rceil\right\}.
\end{equation*}
To give a sense of their magnitudes, the edge length of the bins at the maximal level is approximately $T^{-1/(d+4)}$.
The range of the decision set ($\Delta_k$) is proportional to the edge length of the bin ($2^{-k}$).
The number of prices in a decision set is approximately $\log(T)$.
Therefore, the grid size is $\delta_B\approx 2^{-k}$ for a level-$k$.
The number of customers to observe in a level-$k$ bin $B$ is roughly $n_k\approx 2^{4k}/\log(T)^2$ before it splits.
When $k$ is small, $n_k$ can be zero according to the expression. In this case, the algorithm immediately splits the bin without collecting any sales data in it.

\subsection{A Schematic Illustration}
We illustrate the key steps of the algorithm by an example with $d=2$.
Figure~\ref{fig:illustration} illustrates a possible outcome of the algorithm in periods
$t_1<t_2<t_3$ (top panel, mid panel, and bottom panel respectively).
Up until period $t_1$, there is a single bin and the covariates of observed customers $\bm X_t$ for $t\le t_1$ are illustrated in the top left panel.
In this case, the decision set associated with the root bin is $p\in\{0.1,0.2,\dots,0.9\}$, illustrated by the top right panel.
The average revenue $\bar Y_{B,j}$ of each price is recorded, and $p^{\ast}=0.6$ is the empirically-optimal price.
At $t_1+1$, a sufficient number of customers are observed and Step~\ref{step:enough-data} is triggered in the algorithm.
Therefore, the bin is split into four child bins.

From period $t_1+1$ to $t_2$, new customers are observed in each child bin (mid left panel).
Note that the customers observed before $t_1$ in the parent bin are no longer used and colored in gray.
For each child bin (the bottom-left bin is abbreviated to BL, etc.), the average revenues for the prices in the decision sets is demonstrated in the mid right panel.
The decision sets are centered at the empirically-optimal price of their parent bin, which is $p^{\ast}=0.6$ from the top right panel.
They have narrower ranges and finer grids than that of the parent bin.
At $t_2+1$, a sufficient number of customers are observed in BL, and it is split into four child bins.

From period $t_2+1$ to $t_3$, the partition consists of seven bins, as shown in the bottom left panel.
The BR, TL and TR bins keep observing customers and updating the average revenues, because they have not collected sufficient data.
Their status at $t_3$ is shown in the bottom panels.
In the four newly created child bins of BL (the bottom-left bin of BL is abbreviated to BL-BL, etc.),
the prices in the decision sets are used successively
and their average revenues are illustrated in the bottom right panel.
\begin{figure}[]
    \centering
    \includegraphics{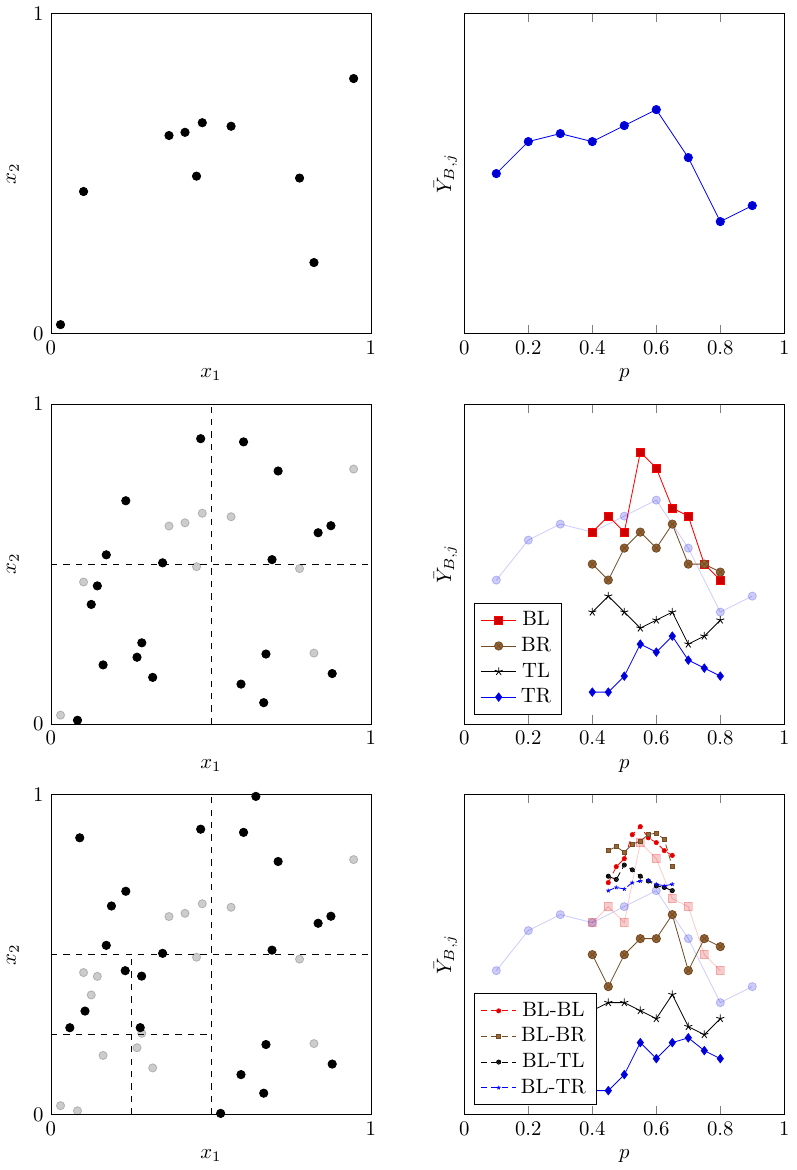}
    \caption{A schematic illustration of the ABE algorithm. }%The left panels may have only shown a proportion of the observed covariates.}
    \label{fig:illustration}
\end{figure}

\section{Regret Analysis: Upper Bound}\label{sec:upper}
To measure the performance of the ABE algorithm,
we provide an upper bound for its regret.
\begin{theorem}\label{thm:upper-bound}
    For any function $f$ satisfying Assumption~\ref{asp:lipschitz} and~\ref{asp:maximizer}, the regret incurred by the ABE algorithm is bounded by
    \begin{equation*}
        R_{\pi_{ABE}}(T)\le C T^{ \frac{2+d}{4+d}}\log(T)^2
    \end{equation*}
    for a constant $C>0$ that is independent of $T$.
\end{theorem}
We provide a sketch of the proof here and present the details in the appendix.
In period $t$, if $\bm X_t\in B$ for a bin in the partition $B\in\mathcal P_t$, then the expected regret incurred by the ABE algorithm is $\E[(f^{\ast}(\bm X_t)-f(\bm X_t,p_t))\I{\bm X_t\in B,B\in \mathcal P_t}]$.
Since the total regret simply sums up the above quantity over $t=1,\dots,T$ and all possible $B$s,
it suffices to focus on the regret for given $t$ and $B$.
Two possible scenarios can arise:
(1) the optimal price of the aggregate demand in $B$, i.e., $p^{\ast}(B)$, is inside the range of the decision set, i.e., $p^{\ast}(B)\in [p^B_l,p^B_u]$ (Step~\ref{step:child-interval});
(2) the optimal price $p^{\ast}(B)$ is outside the range of the decision set.

Scenario one represents the regime where the algorithm is working ``normally'': up until $t$, the algorithm has successfully narrowed the optimal price $p^{\ast}(B)$ (which provides a useful indicator for all $p^{\ast}(\bm x)$, $\bm x\in B$ when $B$ is small) down to $[p^B_l,p^B_u]$.
By Assumption~\ref{asp:maximizer} part one, the regret in this scenario can be decomposed into two terms
\begin{equation*}
    f^{\ast}(\bm X_t)-f(\bm X_t,p_t)\le M_3(|p_t-p^{\ast}(B)|+|p^{\ast}(B)-p^{\ast}(\bm X_t)|)^2.
\end{equation*}
The first term can be bounded by the length of the interval $p^B_u-p^B_l$.
The second term can be bounded by the size of $B$ given $\bm X_t\in B$ by Assumption~\ref{asp:maximizer} part two and three.
By the choice of parameters in Section~\ref{sec:choice}, the length of the interval decreases as the bin size decreases.
Therefore, both terms can be well controlled when the size of $B$ is sufficiently small, or equivalently, when the level $l(B)$ is sufficiently large.
This is why a properly chosen $n_k$ can guarantee that the algorithm spends little time for large bins and collect a large amount of data for small bins.
When the bin level reaches $K$, the above two terms are small enough and no more exploration is needed.
%The former regret can be bounded by the range of the decision set $p^B_u-p_l^B$ because of the fact that $|\pi_t-p^{\ast}(B)|\le p^B_u-p_l^B$ and Assumption~\ref{asp:maximizer} part 1.
%The latter regret can be bounded by Assumption~\ref{asp:maximizer} part 3 when the size of $B$ is sufficiently small, or equivalently, the level of $B$ is sufficiently large.
%In both cases, the regret decreases substantially as the level of the bin increases (which also implies the shortening of the range of the decision set).
%The choice of

Scenario two represents the regime where the algorithm works ``abnormally''.
In scenario two, the difference $f^{\ast}(\bm X_t)-f(\bm X_t,p_t)$ can no longer be controlled as in scenario one because $p_t$ and $p^{\ast}(\bm X_t)$ can be far apart.
To make things worse, $p^{\ast}(B)\notin [p^B_l,p^B_u]$ usually implies $p^{\ast}(B')\notin [p^{B'}_l,p^{B'}_u]$, where $B'$ is a child of $B$.
This is because (1) $p^{\ast}(B)$ is close to $p^{\ast}(B')$ for small $B$, and (2)
$[p^{B'}_l,p^{B'}_u]$ is created around the empirically-optimal price for $B$, and thus overlapping with $[p^{B}_l,p^{B}_u]$.
%\begin{itemize}
%    \item The decision set of any offspring of $B$, tends to be close to, or even contained in the decision set of $B$ as the algorithm constructs the decision set of child bins centered at the empirically-optimal decision for the parent bin (Step~\ref{step:child-interval});
%    \item $p^{\ast}(B)$ is likely to be close to $p^{\ast}(B')$ for small $B$ by Assumption~\ref{asp:maximizer}.
%\end{itemize}
%imply that $p^{\ast}(B')$ is likely to be outside $[p_l^{B'},p_u^{B'}]$ as well.
Therefore, for any period $s$ following $t$, the worst-case regret is $O(1)$ in that period if $\bm X_s\in B$ or its offspring.

To bound the regret in scenario two, we have to bound the probability,
which requires delicate analysis of the events.
If scenario two occurs for $B$, then during the process that we sequentially split $B_{\emptyset}$ to obtain $B$, we can find an ancestor bin of $B$ (which can be $B$ itself) that scenario two happens for the first time along the ``branch'' from $B_{\emptyset}$ all the way down to $B$.
More precisely, denoting the ancestor bin by $B_a$ and its parent by $P(B_a)$, we have
(1) $p^{\ast}(P(B_a))$ is inside $[p^{P(B_a)}_l,p^{P(B_a)}_u]$ (scenario one);
(2) after $P(B_a)$ is split, $p^{\ast}(B_a)$ is outside $[p^{B_a}_l,p^{B_a}_u]$  (scenario two).
Denote the empirically-optimal price in the decision set of $P(B_a)$ by $p^{\ast}$.
For such an event to occur, the center of the decision set of $B_a$, which is $p^{\ast}$, has to be at least $\Delta_{l(B_a)}/2$ away from $p^{\ast}(B_a)$.\footnote{Recall that $p^{B_a}_u-p^{B_a}_l=\Delta_{l(B_a)}$.}
Because of Assumption~\ref{asp:maximizer} and the choice of $\Delta_k$, the distance between $p^{\ast}(P(B_a))$ and $p^{\ast}(B_a)$ is relatively small compared to $\Delta_{l(B_a)}$.
Therefore, the empirically-optimal price $p^{\ast}$ must be far away from $p^{\ast}(P(B_a))$.
The probability of such event can be bounded using classic concentration inequalities for sub-Gaussian random variables:
the prices that are closer to $p^{\ast}(P(B_a))$ and thus have higher means turn out to generate lower average revenue than $p^{\ast}$;
this event is extremely unlikely to happen when we have collected a large amount of sales data for each price in the decision set.

The total regret aggregates those in scenario one and two for all possible combinations of $B$ and $B_a$.
It matches the quantity $O(\log(T)^2T^{(2+d)/(4+d)})$ presented in the theorem.

\section{Regret Analysis: Lower Bound}\label{sec:lower-bound}
In this section, we show that the minimax regret is no lower than $cT^{(2+d)/(4+d)}$ for some constant $c$.
Combining with the last section, we conclude that no non-anticipating policy does better than the ABE algorithm in terms of the order of magnitude of the regret in $T$ (neglecting logarithmic terms).

We first construct a family of functions that satisfy Assumption~\ref{asp:lipschitz} and~\ref{asp:maximizer}.
The functions in the family are selected to be ``difficult'' to distinguish.
By doing so, we will prove that any policy has to spend a substantial amount of time exploring prices that generate low revenues but help to differentiate the functions.
Otherwise, the incapability to correctly identify the underlying function is costly in the long run.
Therefore, unable to contain both sources of regret at the same time, no policy can achieve lower regret than the quantity stated in Theorem~\ref{thm:lower-bound}.
%Similar ideas have been proposed in \citet{rigollet2010nonparametric}.

Before introducing the family of functions, we define $\partial B$ to be the boundary of a convex set in $[0,1)^d$.
Let $D(B_1,B_2)$ be the Euclidean distance between two sets $B_1$ and $B_2$.
That is $D(B_1,B_2)\triangleq \inf\left\{\|\bm x_1-\bm x_2\|_2:\bm x_1\in B_1,\bm x_2\in B_2\right\}$.
We allow $B_1$ or $B_2$ to be a singleton.
To define $\mathcal C$, we partition the covariate space $[0,1)^d$ into $M^d$ equally sized bins.
That is, each bin has the following form: for $(k_1,\dots,k_d)\in\{1,\dots,M\}^d$,
\begin{equation*}
    \left\{\bm x: \frac{k_i-1}{M}\le x_i< \frac{k_i}{M},\quad  \forall\; i=1,\dots,d\right\}.
\end{equation*}
We number those bins by $1,\dots, M^d$ in an arbitrary order, i.e., $B_1,\dots,B_{M^d}$.
Each function $f(\bm x,p)\in\mathcal C$ is indexed by a tuple $w\in \{0,1\}^{M^d}$, whose
$j$th index determines the behavior of $f_w(\bm x,p)$ in $ B_j$.
More precisely, for $\bm x\in B_j$,
%\new{we transform the original lower bound $f=-p^2$ and $f=-p^2+2pD(\bm x,\partial B_j)$ by $p\to (1-p)$ and make $f\to f/2+1$, the noise is changed from normal to Bernoulli}
the personalized demand function is
\begin{align*}
    d_w(\bm x,p)= \begin{cases}
        \frac{2}{3}- \frac{p}{2} & w_j=0\\
        \frac{2}{3}- \frac{p}{2}+ \left( \frac{1}{3}- \frac{p}{2}\right)D(\bm x,\partial B_j) & w_j=1
    \end{cases}
\end{align*}
and thus
\begin{align*}
    f_w(\bm x,p)= \begin{cases}
        p\left(\frac{2}{3}- \frac{p}{2}\right) & w_j=0\\
        p\left(\frac{2}{3}- \frac{p}{2}+ \left( \frac{1}{3}- \frac{p}{2}\right)D(\bm x,\partial B_j)\right)& w_j=1
    \end{cases}
\end{align*}
The optimal personalized price for customer $\bm x\in B_j$ is $p^{\ast}(\bm x)=2/3$ if $w_j=0$ and $p^{\ast}(\bm x)= \frac{2+D(\bm x,\partial B_j)}{3(1+D(\bm x,\partial B_j))}$ if $w_j=1$.
%In other words, when $w_j=1$, the optimal price $p^{\ast}(\bm x)$ is zero when $\bm x$ is at the boundary of $B_j$; as $\bm x$ moves away from the boundary, $p^{\ast}(\bm x)$ decreases and reaches its minimum $1/2M$ when $\bm x$ is at the center of $B_j$.

The construction of $\mathcal C$ follows a similar idea to \citet{rigollet2010nonparametric}.
For a given $f_w\in \mathcal C$, we can always find another function $f_{w'}\in\mathcal C$ that only differs from $f$ in a single bin $B_j$ by setting $w'$ to be equal to $w$ except for the $j$th index.
The firm can only rely on the covariates generated in $B_j$ to distinguish between $f_w$ and $f_{w'}$.
For small bins (i.e., large $M$), this is particularly costly because there are only a tiny fraction of customers observed in a particular bin and the difference $|f_w-f_{w'}|=p(1/3-p/2)D(\bm x,\partial B_j)$ becomes tenuous.
It also requires $p$ to be far from $2/3$ to detect the difference, which happens to be the optimal price when $w_j=0$.
This makes the exploration/exploitation trade-off hard to balance.
Now a policy has to carry out the task for $M^d$ bins, i.e., distinguishing the underlying function $f_w$ with $M^d$ tuples that only differ from $w$ in one index.
The cost is inevitable and adds to the lower bound of the regret.
Moreover, we assume the customers $\bm X$ are uniformly distributed in $[0,1)^d$.

In the appendix, we show that the constructed $\mathcal C$ satisfies all the assumptions.
The main theorem below shows the lower bound for the regret.
\begin{theorem}\label{thm:lower-bound}
    For the constructed $\mathcal C$, any non-anticipating policy $\pi$ has regret
    \begin{equation*}
        \sup_{f\in \mathcal C}R_\pi(T)\ge c T^{\frac{2+d}{4+d}}
    \end{equation*}
    for a constant $c>0$.
\end{theorem}

\section{Future Research}
%In this section, we discuss potential research directions.
%\subsection{Adaptive Binning}
%In the ABE algorithm, the covariate space is refined adaptively: a bin is split only when we have observed a sufficient number of customers in it.
%An alternative idea of binning, similar to the algorithm designed in \citet{rigollet2010nonparametric}, is to pre-define a set of static bins that are sufficiently small.
%This is equivalent to categorizing customers in various types in advance, and performs parallel learning for each type.
%When binning adaptively, a customer and the observed revenue in a parent bin provide some information for all offspring bins.
%Such information pooling makes the exploration more effective.
%While in static binning, the information learned from a customer is only restricted to its own type.

%\subsection{The Rate of Regret and Sparsity}\label{sec:sparsity}
As shown in Theorem~\ref{thm:upper-bound} and Theorem~\ref{thm:lower-bound},
the best achievable regret of the problem is of order $T^{(2+d)/(4+d)}$.
%Note that the parametric version of the problem (e.g., \citealt{ban2017personalized}) can achieve regret of order $\sqrt{T}$ or even $\log(T)$ under certain conditions (e.g., \citealt{qiang2016dynamic,javanmard2016dynamic}), which is the same order as the problem without covariates (e.g., \citealt{denboer2014simul}).
%In other words, when the covariate is parametric (linear),
%it does not complicates the learning problem, regardless of the dimension of the covariate.
%In our nonparametric formulation, for large $d$, no algorithms can achieve regret of order less than $T^{(2+d)/(4+d)}$.
%It implies that when the dimension of the covariate grows, the learning problem becomes quite intractable and the incurred regret is almost linear in $T$, which is considered unsatisfactory because an arbitrary policy can achieve linear regret.
As a result, the knowledge of the sparsity structure of the covariate is essential in designing pricing policies.
More precisely, the provided customer covariate is of dimension $d$, while the personalized demand $d(\bm X,p)$ may only depend on $d'$ entries of the covariate where $d'\ll d$.
In this case, being able to identify the $d'$ entries out of $d$ significantly decreases the incurred regret from $T^{(2+d)/(4+d)}$ to $T^{(2+d')/(4+d')}$.
Indeed, in the ABE algorithm, if the sparsity structure is known, then a bin is split into $2^{d'}$ instead of $2^d$ child bins.
It pools the observations that only differ in the dimensions corresponding to the redundant covariates so that more observations are available in a bin, and thus substantially reduces the exploration cost.
An important research question is then whether it is possible to
%To circumvent the sparsity problem, one potential method to improve the ABE algorithm is to introduce regression/classification decision trees \citep{hastie01statisticallearning}.
%More precisely, when a bin is split, instead of bisecting all $d$ dimensions in the middle and thus obtaining $2^d$ child bins,
design a binning algorithm that selects one dimension and the position to split, based on a certain criterion, like regression/classification decision trees \citep{hastie01statisticallearning}.
%For example, we may select a dimension and the associated cutoff so that the sum of the sample standard deviations in the two resulting child bins is minimized.
%or equivalently, the decrease in the total in-bin standard deviation of the current partition is maximized.
This may significantly improve the regret in the presence of sparse covariates.
%In Step~\ref{step:partition-update} of Algorithm~\ref{alg:abe}, the split only increases the size of the partition by one instead of $2^d-1$.
%The pooling of past observations can lead to more effective explorations.
%The detail of this algorithm and its regret analysis remain a topic for future research.

%\section*{Acknowledgement}
%Ningyuan Chen thanks the support of the Early Career Award 26201617 of Hong Kong Research Grant Concil and the startup funds of the University of Toronto.
%Guillermo Gallego thanks the support of research grant under project number 16204117 from the Hong Kong Research Grant Council.

\bibliography{ref.bib}

\begin{thebibliography}{}

\bibitem[\protect\citeauthoryear{Agrawal}{Agrawal}{1995}]{agrawal1995continuum}
Agrawal, R. (1995).
\newblock The continuum-armed bandit problem.
\newblock {\em SIAM journal on control and optimization\/}~{\em 33\/}(6),
  1926--1951.

\bibitem[\protect\citeauthoryear{Agrawal and Goyal}{Agrawal and
  Goyal}{2012}]{agrawal2012analysis}
Agrawal, S. and N.~Goyal (2012).
\newblock Analysis of thompson sampling for the multi-armed bandit problem.
\newblock In {\em Conference on Learning Theory}, pp.\  39--1.

\bibitem[\protect\citeauthoryear{Araman and Caldentey}{Araman and
  Caldentey}{2009}]{araman2009dynamic}
Araman, V.~F. and R.~Caldentey (2009).
\newblock Dynamic pricing for nonperishable products with demand learning.
\newblock {\em Operations research\/}~{\em 57\/}(5), 1169--1188.

\bibitem[\protect\citeauthoryear{Auer, Ortner, and Szepesv{\'a}ri}{Auer
  et~al.}{2007}]{Auer2007}
Auer, P., R.~Ortner, and C.~Szepesv{\'a}ri (2007).
\newblock {\em Improved Rates for the Stochastic Continuum-Armed Bandit
  Problem}, pp.\  454--468.
\newblock Berlin, Heidelberg: Springer Berlin Heidelberg.

\bibitem[\protect\citeauthoryear{Ban and Keskin}{Ban and
  Keskin}{2017}]{ban2017personalized}
Ban, G. and N.~B. Keskin (2017).
\newblock Personalized dynamic pricing with machine learning.
\newblock {\em Working paper\/}.

\bibitem[\protect\citeauthoryear{Bastani and Bayati}{Bastani and
  Bayati}{2015}]{bastani2015online}
Bastani, H. and M.~Bayati (2015).
\newblock Online decision-making with high-dimensional covariates.
\newblock {\em Working paper\/}.

\bibitem[\protect\citeauthoryear{Besbes and Zeevi}{Besbes and
  Zeevi}{2009}]{besbes2009dynamic}
Besbes, O. and A.~Zeevi (2009).
\newblock Dynamic pricing without knowing the demand function: Risk bounds and
  near-optimal algorithms.
\newblock {\em Operations Research\/}~{\em 57\/}(6), 1407--1420.

\bibitem[\protect\citeauthoryear{Besbes and Zeevi}{Besbes and
  Zeevi}{2012}]{besbes2012blind}
Besbes, O. and A.~Zeevi (2012).
\newblock Blind network revenue management.
\newblock {\em Operations research\/}~{\em 60\/}(6), 1537--1550.

\bibitem[\protect\citeauthoryear{Broder and Rusmevichientong}{Broder and
  Rusmevichientong}{2012}]{broder2012dynamic}
Broder, J. and P.~Rusmevichientong (2012).
\newblock Dynamic pricing under a general parametric choice model.
\newblock {\em Operations Research\/}~{\em 60\/}(4), 965--980.

\bibitem[\protect\citeauthoryear{Bubeck and Cesa-Bianchi}{Bubeck and
  Cesa-Bianchi}{2012}]{bubeck2012regret}
Bubeck, S. and N.~Cesa-Bianchi (2012).
\newblock Regret analysis of stochastic and nonstochastic multi-armed bandit
  problems.
\newblock {\em Foundations and Trends{\textregistered} in Machine
  Learning\/}~{\em 5\/}(1), 1--122.

\bibitem[\protect\citeauthoryear{Bubeck, Munos, Stoltz, and
  Szepesv{\'a}ri}{Bubeck et~al.}{2011}]{bubeck2011x}
Bubeck, S., R.~Munos, G.~Stoltz, and C.~Szepesv{\'a}ri (2011).
\newblock X-armed bandits.
\newblock {\em Journal of Machine Learning Research\/}~{\em 12\/}(May),
  1655--1695.

\bibitem[\protect\citeauthoryear{Cesa-Bianchi and Lugosi}{Cesa-Bianchi and
  Lugosi}{2006}]{cesa2006prediction}
Cesa-Bianchi, N. and G.~Lugosi (2006).
\newblock {\em Prediction, learning, and games}.
\newblock Cambridge university press.

\bibitem[\protect\citeauthoryear{Chen and Gallego}{Chen and
  Gallego}{2018}]{chen2018primal}
Chen, N. and G.~Gallego (2018).
\newblock A primal-dual learning algorithm for personalized dynamic pricing
  with an inventory constraint.
\newblock {\em Working paper\/}.

\bibitem[\protect\citeauthoryear{Cheung, Simchi-Levi, and Wang}{Cheung
  et~al.}{2017}]{cheung2017dynamic}
Cheung, W.~C., D.~Simchi-Levi, and H.~Wang (2017).
\newblock Dynamic pricing and demand learning with limited price
  experimentation.
\newblock {\em Operations Research\/}~{\em 65\/}(6), 1722--1731.

\bibitem[\protect\citeauthoryear{Cohen, Lobel, and Paes~Leme}{Cohen
  et~al.}{2016}]{cohen2016feature}
Cohen, M.~C., I.~Lobel, and R.~Paes~Leme (2016).
\newblock Feature-based dynamic pricing.
\newblock {\em Working paper\/}.

\bibitem[\protect\citeauthoryear{den Boer}{den Boer}{2015}]{den2015dynamic}
den Boer, A.~V. (2015).
\newblock Dynamic pricing and learning: historical origins, current research,
  and new directions.
\newblock {\em Surveys in operations research and management science\/}~{\em
  20\/}(1), 1--18.

\bibitem[\protect\citeauthoryear{den Boer and Zwart}{den Boer and
  Zwart}{2014}]{denboer2014simul}
den Boer, A.~V. and B.~Zwart (2014).
\newblock Simultaneously learning and optimizing using controlled variance
  pricing.
\newblock {\em Management Science\/}~{\em 60\/}(3), 770--783.

\bibitem[\protect\citeauthoryear{Elmachtoub, McNellis, Oh, and
  Petrik}{Elmachtoub et~al.}{2017}]{elmachtoub2017practical}
Elmachtoub, A.~N., R.~McNellis, S.~Oh, and M.~Petrik (2017).
\newblock A practical method for solving contextual bandit problems using
  decision trees.
\newblock {\em Working paper\/}.

\bibitem[\protect\citeauthoryear{Farias and Van~Roy}{Farias and
  Van~Roy}{2010}]{farias2010dynamic}
Farias, V.~F. and B.~Van~Roy (2010).
\newblock Dynamic pricing with a prior on market response.
\newblock {\em Operations Research\/}~{\em 58\/}(1), 16--29.

\bibitem[\protect\citeauthoryear{Foucart and Rauhut}{Foucart and
  Rauhut}{2013}]{foucart2013mathematical}
Foucart, S. and H.~Rauhut (2013).
\newblock {\em A mathematical introduction to compressive sensing}, Volume~1.
\newblock Birkh{\"a}user Basel.

\bibitem[\protect\citeauthoryear{Gallego and Topaloglu}{Gallego and
  Topaloglu}{2018}]{ggbook}
Gallego, G. and H.~Topaloglu (2018).
\newblock Revenue management and pricing analytics.
\newblock In preparation.

\bibitem[\protect\citeauthoryear{Goldenshluger and Zeevi}{Goldenshluger and
  Zeevi}{2013}]{goldenshluger2013linear}
Goldenshluger, A. and A.~Zeevi (2013).
\newblock A linear response bandit problem.
\newblock {\em Stochastic Systems\/}~{\em 3\/}(1), 230--261.

\bibitem[\protect\citeauthoryear{Goldenshluger, Zeevi, et~al.}{Goldenshluger
  et~al.}{2009}]{goldenshluger2009woodroofe}
Goldenshluger, A., A.~Zeevi, et~al. (2009).
\newblock Woodroofe’s one-armed bandit problem revisited.
\newblock {\em The Annals of Applied Probability\/}~{\em 19\/}(4), 1603--1633.

\bibitem[\protect\citeauthoryear{Hastie, Tibshirani, and Friedman}{Hastie
  et~al.}{2001}]{hastie01statisticallearning}
Hastie, T., R.~Tibshirani, and J.~Friedman (2001).
\newblock {\em The Elements of Statistical Learning}.
\newblock Springer Series in Statistics. New York, NY, USA: Springer New York
  Inc.

\bibitem[\protect\citeauthoryear{Javanmard and Nazerzadeh}{Javanmard and
  Nazerzadeh}{2016}]{javanmard2016dynamic}
Javanmard, A. and H.~Nazerzadeh (2016).
\newblock Dynamic pricing in high-dimensions.
\newblock {\em Working paper\/}.

\bibitem[\protect\citeauthoryear{Keskin and Zeevi}{Keskin and
  Zeevi}{2014}]{keskin2014dynamic}
Keskin, N.~B. and A.~Zeevi (2014).
\newblock Dynamic pricing with an unknown demand model: Asymptotically optimal
  semi-myopic policies.
\newblock {\em Operations Research\/}~{\em 62\/}(5), 1142--1167.

\bibitem[\protect\citeauthoryear{Kleinberg, Slivkins, and Upfal}{Kleinberg
  et~al.}{2008}]{kleinberg2008multi}
Kleinberg, R., A.~Slivkins, and E.~Upfal (2008).
\newblock Multi-armed bandits in metric spaces.
\newblock In {\em Proceedings of the fortieth annual ACM symposium on Theory of
  computing}, pp.\  681--690. ACM.

\bibitem[\protect\citeauthoryear{Kleinberg}{Kleinberg}{2005}]{kleinberg2005nearly}
Kleinberg, R.~D. (2005).
\newblock Nearly tight bounds for the continuum-armed bandit problem.
\newblock In {\em Advances in Neural Information Processing Systems}, pp.\
  697--704.

\bibitem[\protect\citeauthoryear{Kuleshov and Precup}{Kuleshov and
  Precup}{2014}]{kuleshov2014algorithms}
Kuleshov, V. and D.~Precup (2014).
\newblock Algorithms for multi-armed bandit problems.
\newblock {\em Working paper\/}.

\bibitem[\protect\citeauthoryear{Langford and Zhang}{Langford and
  Zhang}{2008}]{langford2008epoch}
Langford, J. and T.~Zhang (2008).
\newblock The epoch-greedy algorithm for multi-armed bandits with side
  information.
\newblock In {\em Advances in neural information processing systems}, pp.\
  817--824.

\bibitem[\protect\citeauthoryear{Lei, Jasin, and Sinha}{Lei
  et~al.}{2017}]{lei2014near}
Lei, Y., S.~Jasin, and A.~Sinha (2017).
\newblock Near-optimal bisection search for nonparametric dynamic pricing with
  inventory constraint.
\newblock {\em Working paper\/}.

\bibitem[\protect\citeauthoryear{Nambiar, Simchi-Levi, and Wang}{Nambiar
  et~al.}{2016}]{nambiar2016dynamic}
Nambiar, M., D.~Simchi-Levi, and H.~Wang (2016).
\newblock Dynamic learning and price optimization with endogeneity effect.
\newblock {\em Working paper\/}.

\bibitem[\protect\citeauthoryear{Perchet and Rigollet}{Perchet and
  Rigollet}{2013}]{perchet2013}
Perchet, V. and P.~Rigollet (2013).
\newblock The multi-armed bandit problem with covariates.
\newblock {\em The Annals of Statistics\/}~{\em 41\/}(2), 693--721.

\bibitem[\protect\citeauthoryear{Qiang and Bayati}{Qiang and
  Bayati}{2016}]{qiang2016dynamic}
Qiang, S. and M.~Bayati (2016).
\newblock Dynamic pricing with demand covariates.
\newblock {\em Working paper\/}.

\bibitem[\protect\citeauthoryear{Rigollet and Zeevi}{Rigollet and
  Zeevi}{2010}]{rigollet2010nonparametric}
Rigollet, P. and A.~Zeevi (2010).
\newblock Nonparametric bandits with covariates.
\newblock In A.~T. Kalai and M.~Mohri (Eds.), {\em COLT}, pp.\  54--66.
  Omnipress.

\bibitem[\protect\citeauthoryear{Slivkins}{Slivkins}{2014}]{slivkins2014contextual}
Slivkins, A. (2014).
\newblock Contextual bandits with similarity information.
\newblock {\em The Journal of Machine Learning Research\/}~{\em 15\/}(1),
  2533--2568.

\bibitem[\protect\citeauthoryear{Tsybakov}{Tsybakov}{2009}]{tsybakov2009introduction}
Tsybakov, A.~B. (2009).
\newblock {\em Introduction to Nonparametric Estimation\/} (1 ed.).
\newblock Springer-Verlag New York.

\bibitem[\protect\citeauthoryear{Tsybakov et~al.}{Tsybakov
  et~al.}{2004}]{tsybakov2004optimal}
Tsybakov, A.~B. et~al. (2004).
\newblock Optimal aggregation of classifiers in statistical learning.
\newblock {\em The Annals of Statistics\/}~{\em 32\/}(1), 135--166.

\bibitem[\protect\citeauthoryear{Wang, Deng, and Ye}{Wang
  et~al.}{2014}]{wang2014close}
Wang, Z., S.~Deng, and Y.~Ye (2014).
\newblock Close the gaps: A learning-while-doing algorithm for single-product
  revenue management problems.
\newblock {\em Operations Research\/}~{\em 62\/}(2), 318--331.

\bibitem[\protect\citeauthoryear{Yang, Zhu, et~al.}{Yang
  et~al.}{2002}]{yang2002randomized}
Yang, Y., D.~Zhu, et~al. (2002).
\newblock Randomized allocation with nonparametric estimation for a multi-armed
  bandit problem with covariates.
\newblock {\em The Annals of Statistics\/}~{\em 30\/}(1), 100--121.

\end{thebibliography}

\singlespacing
\newpage
\appendix

\begin{center}
    {\large\textbf{Online Appendix for Nonparametric Pricing Analytics with Customer Covariates}}
\end{center}
\section{Table of Notations}
\begin{table}[h]
    \centering
    \begin{tabular}{cl}
        \toprule
        $\lceil x\rceil$& The smallest integer that does not exceed $x$\\
        $( x)^+$& The positive part of $x$\\
        $\#\{\}$& The cardinality of a set\\
        $\mathcal F_t$& The $\sigma$-algebra generated by $(X_1,\pi_1,Z_1,\dots,X_{t-1},\pi_{t-1},Z_{t-1})$\\
        $\mu_{\bm X}$ & The distribution of the covariate $\bm X$ over $[0,1)^d$\\
        $P^{\ast}(B)$ & $\argmax_{p\in[0,1]}\left\{\E[f(\bm X,p)|\bm X\in B]\right\}$\\
        $P^{\ast}_B$ & The empirically-optimal decision in the decision set for $B$\\
        $\partial B$ & The boundary of $B$\\
        \bottomrule
    \end{tabular}
    \caption{A table of notations used in the paper.}
    \label{tab:notation}
\end{table}
\section{Proofs}
\begin{proof}[Proof of Proposition~\ref{prop:quadratic}:]
    Define for $p\in[0,1]$
    \begin{equation*}
        g(p) = \begin{cases}{}
            \frac{f_B(p^{\ast}(B))-f_B(p)}{(p^{\ast}(B)-p)^2} & p\neq p^{\ast}(B)\\
                                -\frac{f''_B(p^{\ast}(B))}{2} & p = p^{\ast}(B).
        \end{cases}
    \end{equation*}
    By L'Hopital's rule, $g(p)$ is continuous at $p^{\ast}(B)$.
    In addition, because $f_B(p)$ is continuous, $g(p)$ is continuous for all $p\in[0,1]$.
    By Weierstrass's extreme value theorem, we have $g(p)\in[M_2, M_3]$ and both $M_2$ and $M_3$ are attained.
    By the definition and the uniqueness of the maximizer, $g(p)>0$ for $p\in[0,1]$.
    Therefore, we must have $M_2,M_3>0$.
    This establishes the result.
%
%    Because $f_B(p)$ is twice differentiable at $p^\ast(B)$,
%    there exists a $\delta>0$ such that for $p\in(p^{\ast}(B)-\delta,p^{\ast}(B)+\delta)$,  $f_B(p)$ is differentiable and moreover
%    \begin{align*}
%        \frac{f'_B(p)-f'_B(p^{\ast}(B))}{p-p^{\ast}(B)} =  \frac{f'_B(p)}{p-p^{\ast}(B)}\in\left(2 f''_B(p^{\ast}(B)), \frac{1}{2}f''_B(p^{\ast}(B)) \right).
%    \end{align*}
%    By the Taylor expansion, there exists a $\bar p$ in that neighborhood such that
%    \begin{align*}
%        f_B(p^{\ast}(B))-f_B(p) &= f_B'(\bar p)(p^{\ast}(B)-\bar p) \\
%                                &\in\left(- \frac{1}{2}f''_B(p^{\ast}(B))(p^{\ast}(B)-\bar p)^2, -2 f''_B(p^{\ast}(B))(p^{\ast}(B)-\bar p)^2  \right).
%    \end{align*}
%    Therefore, part one of Assumption~\ref{asp:maximizer} holds in a small neighborhood of $p^{\ast}(B)$.
%
%    For $p$ outside the neighborhood,
%    note that by Lipschitz continuity, $f_B(p)=\E[f(\bm X,p)|\bm X\in B]$ is also a continuous function for $p\in[0,1]$.
%    Therefore,
%    \begin{equation*}
%      \max_{p\in [0,1]\setminus (p^{\ast}(B)-\delta,p^{\ast}(B)+\delta)} f_B(p)< f_B(p^{\ast}(B))
%    \end{equation*}
%    which implies the existence of $M_2$.
%    On the other hand, $f_B(p)$ is bounded as $p\in[0,1]$ in a compact set.
%    Therefore, one can find $M_4$.
%    This proves the claim.
\end{proof}
\subsection{Upper Bound in Section~\ref{sec:upper}}
We first introduce the following lemmas.
\begin{lemma}\label{lem:bernoulli-subgaussian}
    A Bernoulli random variable $X$ satisfies
    \begin{equation*}
        \E[\exp(t(X-\E[X])]\le \exp(t^2/8)
    \end{equation*}
\end{lemma}
This lemma follows directly from the Hoeffding's lemma. It implies that a Bernoulli random variable is sub-Gaussian.
\begin{lemma}
    \label{lem:sub-g-int}
    Suppose for given $\bm x$ and $p$, the random variable $Z(\bm x,p)$ is sub-Gaussian with parameter $\sigma$, i.e.,
    \begin{equation*}
        \E[\exp(t(Z-\E[Z])]\le \exp(\sigma t^2)
    \end{equation*}
    for all $t\in\R$.
    Then the distribution of $Z(\bm X,p)$ conditional on $\bm X\in B$ for a set $B$ is still sub-Gaussian with the same parameter.
\end{lemma}
\begin{proof}
    Let $\mu_{\bm X}$ denote the distribution of $\bm X$.
    We have that for all $t\in\R$
    \begin{align*}
        &\E[\exp(t(Z(\bm X,p)-\E[Z(\bm X,p)|\bm X\in B])|\bm X\in B]\\
        &= \frac{\int_B \E[\exp(t(Z(\bm x,p)-\E[Z(\bm X,p)|\bm X\in B] d\mu_{\bm X}(\bm x)}{\int_B d\mu_{\bm X}(\bm x)}\\
        &= \frac{\int_B \E[\exp(t(Z(\bm x,p)-\E[Z(\bm x,p)])] d\mu_{\bm X}(\bm x)}{\int_B d\mu_{\bm X}(\bm x)}\times \frac{\int_B\exp(t\E[Z(\bm x,p)])d\mu_{\bm X}(\bm x)}{\exp(t\E[Z(\bm X,p)|\bm X\in B])\int_Bd\mu_{\bm X}(\bm x)}\\
        &\le \frac{\int_B\exp(\sigma t^2)d\mu_{\bm X}(\bm x)}{\int_Bd\mu_{\bm X}(\bm x)}\times 1=\exp(\sigma t^2),
    \end{align*}
    where the last inequality is by the definition of conditional expectations. Hence the result is proved.
\end{proof}
\begin{proof}[Proof of Theorem~\ref{thm:upper-bound}:]
    According to the algorithm (Step~\ref{step:partition}), let $\mathcal P_t$ denote the partition formed by the bins at time $t$ when $\bm X_t$ is generated.
    The regret associated with $\bm X_t$ can be counted by bins $B\in\mathcal P_t$ into which $\bm X_t$ falls.
    Meanwhile, the level of $B$ is at most $K$. Therefore,
    \begin{align*}
        R_{\pi_{ABE}}(T)&=\E\left[\sum_{t=1}^T (f^{\ast}(\bm X_t)-f(\bm X_t,p_t))\right]
                  =\E\left[\sum_{t=1}^T\sum_{B\in \mathcal P_t}(f^{\ast}(\bm X_t)-f(\bm X_t,p_t))\I{\bm X_t\in B}\right]\\
                  &=\E\left[\sum_{t=1}^T\sum_{k=0}^K\sum_{\{B:l(B)=k\}}(f^{\ast}(\bm X_t)-f(\bm X_t,p_t))\I{\bm X_t\in B,B\in\mathcal P_t}\right]
    \end{align*}
    We will define the following random event for each bin $B$:
    \begin{equation*}
        E_B=\left\{p^{\ast}(B)\in\left[p_l^B,p_u^B\right]\right\}.
    \end{equation*}
    Recall that $p^{\ast}(B)$ is the unique maximizer for $f_B(p)=\E[f(\bm X,p)|\bm X\in B]$ by Assumption~\ref{asp:maximizer}; $[p_l^B,p_u^B]$ is the range of the decision set to explore for $B$.
    According to Step~\ref{step:child-interval} of the ABE algorithm, the interval $[p_l^B,p_u^B]$ is constructed around $p^{\ast}_{P(B)}$, the empirically-optimal price of the parent bin $P(B)$ that maximizes the empirical average $\bar Y_{P(B),j}$.

    We will decompose the regret depending on whether $E_B$ occurs.
    \begin{align}
        \E[R_{\pi_{ABE}}]&=\underbrace{\E\left[\sum_{t=1}^T\sum_{k=0}^K\sum_{\{B:l(B)=k\}}(f^{\ast}(\bm X_t)-f(\bm X_t,p_t))\I{\bm X_t\in B,B\in\mathcal P_t,E^c_B}\right]}_{\text{term 1}}\notag\\
                  &\quad +\underbrace{\E\left[\sum_{t=1}^T\sum_{k=0}^K\sum_{\{B:l(B)=k\}}(f^{\ast}(\bm X_t)-f(\bm X_t,p_t))\I{\bm X_t\in B,B\in\mathcal P_t,E_B}\right]}_{\text{term 2}}\label{eq:reg-decomp}
    \end{align}

    We first analyze term 1. Because $E_{B_{\emptyset}}$ is always true ($[p^{B_{\emptyset}}_l,p^{B_{\emptyset}}_u]=[0,1]$ always encloses $p^{\ast}(B_{\emptyset})$ according to Step~\ref{step:initialize}), we can find an ancestor of $B$, say $B_a$ (which can be $B$ itself), such that $E_{B_a}^c\cap E_{P(B_a)}\cap E_{P(P(B_a))}\cap\ldots\cap E_{B_{\emptyset}}$ occurs.
    In other words, up until $B_a$, the algorithm always correctly encloses the optimal price $p^{\ast}(P^{(k)}(B_a))$ of the ancestor bin of $B_a$ in the intervals $\left[p_l^{P^{(k)}(B_a)},p_u^{P^{(k)}(B_a)}\right]$.
    Therefore, when $E_B^c$ occurs,
    we can rearrange the event by such $B_a$.
    Term 1 in \eqref{eq:reg-decomp} can be bounded by
    \begin{align}
        &\E\left[\sum_{t=1}^T\sum_{k=0}^K\sum_{\{B:l(B)=k\}}(f^{\ast}(\bm X_t)-f(\bm X_t,p_t))\I{\bm X_t\in B,B\in\mathcal P_t,E^c_B}\right]\nonumber\\
        &\le \sum_{t=1}^T\sum_{k=0}^K\sum_{\{B:l(B)=k\}}M_1\PR(\bm X_t\in B,B\in\mathcal P_t,E^c_B)\nonumber\\
        &= \sum_{t=1}^T\sum_{k=1}^K\sum_{\{B:l(B)=k\}}M_1\PR\left(\bm X_t\in B,B\in\mathcal P_t,E_{B_a}^c\cap E_{P(B_a)}\cap E_{P(P(B_a))}\cap\ldots\cap E_{B_{\emptyset}}\right)\nonumber\\
        &\le \sum_{t=1}^T\sum_{k=1}^K\sum_{\{B:l(B)=k\}}M_1\PR(\bm X_t\in B,B\in\mathcal P_t,E_{B_a}^c\cap E_{P(B_a)})\nonumber\\
        &= \sum_{t=1}^T\sum_{k=1}^K\sum_{\{B:l(B)=k\}}\sum_{k'=1}^K\sum_{B':l(B')=k'}M_1\PR(\bm X_t\in B,B\in\mathcal P_t,B_a=B',E_{B'}^c\cap E_{P(B')})\label{eq:ancestor-bin}
    \end{align}
    The first inequality is due to Assumption~\ref{asp:lipschitz}.
    In the second inequality, we start enumerating from $k=1$ instead of $k=0$ because $E_{B_{\emptyset}}^c$ never occurs.
    In the last equality, we rearrange the probabilities by counting the deterministic bins $B'$ instead of the random bins $B_a$.

    Now note that $\left\{\bm X_t\in B,B\in\mathcal P_t,B_a=B',E_{B'}^c\cap E_{P(B')}\right\}$ are exclusive for different $B$s because $\mathcal P$ is a partition and $\bm X_t$ can only fall into one bin.
    Moreover, $\{\bm X_t\in B,B\in\mathcal P_t,B_a=B',E_{B'}^c\cap E_{P(B')}\}\subset\left\{\bm X_t\in B',E_{B'}^c\cap E_{P(B')}\right\}$ because $B\subset B_a$.
    Therefore,
    \begin{equation*}
        \sum_{k=1}^K\sum_{\{B:l(B)=k\}}\PR\left(\bm X_t\in B,B\in\mathcal P_t,B_a=B',E_{B'}^c\cap E_{P(B')}\right)\le \PR\left(\bm X_t\in B',E_{B'}^c\cap E_{P(B')}\right).
    \end{equation*}
    Thus, we can further simplify \eqref{eq:ancestor-bin}:
    \begin{align}
        \eqref{eq:ancestor-bin}&\le \sum_{t=1}^T\sum_{k=1}^K\sum_{\{B':l(B')=k\}}M_1\PR\left(\bm X_t\in B',E_{B'}^c\cap E_{P(B')}\right)\notag\\
        &= \sum_{t=1}^T\sum_{k=1}^K\sum_{\{B:l(B)=k\}}M_1\PR(\bm X_t\in B)\PR(E_{B}^c\cap E_{P(B)}).\label{eq:prob-bound1}
    \end{align}
    The last equality is because of the fact that for given $t$ and $B$, the event $\{\bm X_t\in B\}\in\sigma(X_t)$
    %\footnote{We use $\sigma(\cdot)$ to denote the sigma-algebra generated by the random variables.}
    and $E_{B}^c\cap E_{P(B)}\in\mathcal F_{t-1}$.
    Therefore, the two events are independent.

    Next we analyze the event $E_{B}^c\cap E_{P(B)}$ given $l(B)=k$ in order to bound \eqref{eq:prob-bound1}.
    This event implies that when the parent bin $P(B)$ is created, its optimal price $p^{\ast}(P(B))$ is inside the interval $[p_l^{P(B)},p_u^{P(B)}]$.
    At the end of Step~\ref{step:enough-data}, when $n_{k-1}$ customers have been observed in $P(B)$, it is split into $2^d$ children.
    The optimal price of its child bin $B$, that is $p^{\ast}(B)$, is no longer inside $[p_l^{B},p_u^{B}]$.
    By Step~\ref{step:child-interval}, $p_l^B=\max\{p^{\ast}_{P(B)}-\Delta_k/2,0\}$ and $p_u^B=\min\{p^{\ast}_{P(B)}+\Delta_k/2,1\}$, where $p^{\ast}_{P(B)}$ is the empirically-optimal price for $P(B)$.
    Therefore, $E_B^c$ implies that $p^\ast(B)\notin [p^{\ast}_{P(B)}-\Delta_k/2,p^{\ast}_{P(B)}+\Delta_k/2]$.
    Combined with Assumption~\ref{asp:maximizer} part two, which states that $p^{\ast}(B)\in[\inf\{p^{\ast}(\bm x):\bm x\in B\},\sup\{p^{\ast}(\bm x):\bm x\in B\}]\subset \inf\{p^{\ast}(\bm x):\bm x\in P(B)\},\sup\{p^{\ast}(\bm x):\bm x\in P(B)\}$,
    we have
    \begin{equation*}
       [\inf\{p^{\ast}(\bm x):\bm x\in P(B)\},\sup\{p^{\ast}(\bm x):\bm x\in P(B)\}]\nsubset[p^{\ast}_{P(B)}-\Delta_k/2,p^{\ast}_{P(B)}+\Delta_k/2].
    \end{equation*}
    That is, either $\inf\{p^{\ast}(\bm x):\bm x\in P(B)\}<p^{\ast}_{P(B)}-\Delta_k/2$ or $\sup\{p^{\ast}(\bm x):\bm x\in P(B)\}>p^{\ast}_{P(B)}+\Delta_k/2$.
    By Assumption~\ref{asp:maximizer} part three, $\sup\{p^{\ast}(\bm x):\bm x\in P(B)\}-\inf\{p^{\ast}(\bm x):\bm x\in P(B)\}\le M_4d_{P(B)}=M_4 \sqrt{d}2^{-(k-1)}$ because the level of $P(B)$ is $k-1$.
    Hence by Assumption~\ref{asp:maximizer} part two, $\inf\{p^{\ast}(\bm x):\bm x\in P(B)\}\ge p^{\ast}(P(B))-M_4\sqrt{d}2^{-(k-1)}$ and $\sup\{p^{\ast}(\bm x):\bm x\in P(B)\}\le p^{\ast}(P(B))+M_4\sqrt{d}2^{-(k-1)}$.
    Combining the above observations, $E_{B}^c$ could only happen when $|p^{\ast}(P(B))-p^{\ast}_{P(B)}|>\Delta_k/2-M_4\sqrt{d}2^{-(k-1)}$.
    On the other hand, $E_{P(B)}$ implies that $p^{\ast}(P(B))\in [p_l^{P(B)},p_u^{P(B)}]$.
    Therefore, $E_{B}^c\cap E_{P(B)}$ could occur only if there exist two grid points $0\le j_1,j_2\le N_{k-1}-1$ in Step~\ref{step:enough-data} for bin $P(B)$, such that
    \begin{enumerate}
        \item The $j_2$th grid point is the closest to the optimal price for the bin $p^{\ast}(P(B))$. That is, $|p_l^{P(B)}+j_2\delta_{P(B)}-p^{\ast}(P(B))|\le \delta_{P(B)}/2$.
        \item The $j_1$th grid point maximizes $\bar Y_{P(B),j}$. That is $p_l^{P(B)}+j_1\delta_{P(B)}=P^{\ast}_{P(B)}$. It implies that $\bar Y_{P(B),j_1}\ge\bar Y_{P(B),j_2}$ in Step~\ref{step:empirical-mean}.
        \item $|p^{\ast}(P(B))-p^{\ast}_{P(B)}|>\Delta_k/2-M_4\sqrt{d}2^{-(k-1)}$.
    \end{enumerate}
    In other words, the empirically-optimal price is the $j_1$th grid point, while the $j_2$th grid point is closest to the true revenue maximizer in bin $P(B)$, i.e., $p^{\ast}(P(B))$.
    Given that the two grid points are far apart (by point 3 above), the probability of this event should be small.

    To further bound the probability,
    consider $\bar Y_{{P(B)},j}$.
    In Step~\ref{step:empirical-mean}, it is the sum of $\lfloor n_{k-1}/N_{k-1}\rfloor$ or $\lceil n_{k-1}/N_{k-1}\rceil$
    independent random variables with mean
    $\E[f(\bm X,p_l^{P(B)}+j\delta_{P(B)})|\bm X\in P(B)]$.
    By Lemmas~\ref{lem:bernoulli-subgaussian} and~\ref{lem:sub-g-int}, they are still sub-Gaussian with parameter $\sigma=1/8$.
%    Because $|p^{\ast}(P(B_a))-p^{\ast}_{P(B_a)}|>\Delta_k/2-M_4\sqrt{d}2^{-(k-1)}$, it implies that
%    $|(k-j)\delta_k|\ge \Delta_k/2-M_4\sqrt{d}2^{-(k-1)}-\delta_k/2$.
    This gives the following probabilistic bound (recall the definition of $f_B(p)$ in Section~\ref{sec:assumption}):
    \begin{align*}
        \PR(E_{B}^c\cap E_{P(B)})&\le \PR(\bar Y_{P(B),j_1}\ge\bar Y_{P(B),j_2})\\
                                     &= \PR\big( \frac{1}{t_1}\sum_{i=1}^{t_1}X^{(1)}_i- \frac{1}{t_2}\sum_{i=1}^{t_2}X^{(2)}_i\ge f_{P(B)}(p_l^{P(B)}+j_2\delta_{P(B)})-f_{P(B)}(p_l^{P(B)}+j_1\delta_{P(B)})\big)\\
                                     &= \PR\big( \frac{1}{t_1}\sum_{i=1}^{t_1}X^{(1)}_i- \frac{1}{t_2}\sum_{i=1}^{t_2}X^{(2)}_i\ge f_{P(B)}(p_l^{P(B)}+j_2\delta_{P(B)})-f_{P(B)}(p^{\ast}(P(B)))\\
                                     &\quad\quad +f_{P(B)}(p^{\ast}(P(B)))-f_{P(B)}(p_l^{P(B)}+j_1\delta_{P(B)})\big)\\
                                     &\le \PR\big( \frac{1}{t_1}\sum_{i=1}^{t_1}X^{(1)}_i- \frac{1}{t_2}\sum_{i=1}^{t_2}X^{(2)}_i\ge M_2\left(p_l^{P(B)}+j_1\delta_{P(B)}-p^{\ast}(P(B)) \right)^2\\
                                     &\quad\quad-M_3(p_l^{P(B)}+j_2\delta_{P(B)}-p^{\ast}(P(B)))^2\big)\\
                                     &\le \PR\big( \frac{1}{t_1}\sum_{i=1}^{t_1}X^{(1)}_i- \frac{1}{t_2}\sum_{i=1}^{t_2}X^{(2)}_i\ge M_2\left(\left(\Delta_k/2-M_4\sqrt{d}2^{-(k-1)}\right)^+\right)^2\\
                                     &\quad\quad-M_3\delta_{P(B)}^2/4\big).
    \end{align*}
    Here $t_1$ and $t_2$ can be either $\lfloor n_{k-1}/N_{k-1}\rfloor$ or $\lceil n_{k-1}/N_{k-1}\rceil$;
    $X^{(1)}_{i}$ and $X^{(2)}_i$ are independent mean-zero sub-Gaussian random variables with parameter $\sigma$.
    Their averages are the centered version of $\bar Y_{P(B),j_1}$ and $\bar Y_{P(B),j_2}$, and thus their means are moved to the right-hand side.
    In the last inequality, $(\cdot)^+$ represents the positive part. The inequality follows from Assumption~\ref{asp:maximizer} part one and the previously derived facts that
    $|p_l^{P(B)}+j_2\delta_{P(B)}-p^{\ast}(P(B))|\le \delta_{P(B)}/2$ and  $|p_l^{P(B)}+j_1\delta_{P(B)}-p^{\ast}(P(B))|\ge \Delta_k/2-M_4\sqrt{d}2^{-(k-1)}$.
    By the property of sub-Gaussian random variables (for example, see Theorem 7.27 in \citealp{foucart2013mathematical}), the above probability is bounded by
    \begin{align*}
        \PR(E_{B}^c\cap E_{P(B)})&\le \exp\left(- \frac{\left(\left(M_2\left(\left(\Delta_k/2-M_4\sqrt{d}2^{-(k-1)}\right)^+\right)^2-M_3\delta_{P(B)}^2/4\right)^+\right)^2}{4\sigma(1/t_1+1/t_2)}\right)\\
                                     &\le \exp\left(- \frac{n_{k-1}\left(\left(M_2\left(\left(\Delta_k/2-M_4\sqrt{d}2^{-(k-1)}\right)^+\right)^2-M_3\delta_{P(B)}^2/4\right)^+\right)^2}{N_{k-1}+1}\right)
    \end{align*}
    By our choice of parameters, $\Delta_k=2^{-k}\log(T)$, $N_k\equiv \lceil\log(T)\rceil$, $\delta_{P(B)}\le \Delta_{k-1}/N_{k-1}\le 2^{-(k-1)}$.
    Therefore, when $T\ge \max\{\exp(8M_4\sqrt{d}),\exp(4\sqrt{2M_3/M_2})\}$, we have:
    \begin{gather*}
        \frac{\Delta_k}{4}-M_4\sqrt{d}2^{-(k-1)}= 2^{-(k+2)}\log(T)-M_4\sqrt{d}2^{-(k-1)}\ge 0\\
        \Rightarrow \frac{\Delta_k}{2}-M_4\sqrt{d}2^{-(k-1)}\ge\frac{\Delta_k}{4}\\
        \frac{M_2\Delta_k^2}{32}-\frac{M_3 \delta_{P(B)}^2}{4}\ge \frac{M_2 2^{-2k}\log^2(T)}{32}- M_3 2^{-2k}\ge 0\\
        \Rightarrow M_2\big(\Delta_k/2-M_4\sqrt{d}2^{-(k-1)}\big)^2-M_3\delta_{P(B)}^2/4\ge \frac{M_2\Delta_k^2}{16}- \frac{M_3\delta_{P(B)}^2}{4}\ge \frac{M_2\Delta_k^2}{32}.
    \end{gather*}
    Therefore, there exists a constant $c_1=M^2_2/1024$ such that
    \begin{align*}
        \PR(E_{B}^c\cap E_{P(B)})\le \exp\left(- c_1 \frac{\Delta_k^4 n_{k-1}}{\log(T)+1} \right).
    \end{align*}
    With this bound, we can proceed to provide an upper bound for \eqref{eq:prob-bound1}.
    Because $\sum_{\{B:l(B)=k\}}\PR(\bm X_t\in B)=1$, we have
    \begin{align}
        \sum_{t=1}^T\sum_{k=1}^K\sum_{\{B:l(B)=k\}}M_1\PR(\bm X_t\in B)\PR(E_{B}^c\cap E_{P(B)})\le M_1T\sum_{k=1}^K \exp\left(- c_1 \frac{\Delta_k^4 n_{k-1}}{\log(T)+1} \right).\label{eq:term-1}
    \end{align}

    We next analyze term 2 of \eqref{eq:reg-decomp}.
    By Assumption~\ref{asp:maximizer} part one, $f^{\ast}(\bm X_t)-f(\bm X_t,p_t)\le M_3(p^{\ast}(\bm X_t)-p_t)^2\le M_3(|p^{\ast}(B)-p_t|+|p^{\ast}(\bm X_t)-p^{\ast}(B)|)^2$.
    By the design of the algorithm (Step~\ref{step:policy} and \ref{step:policy-K}), $p_t\in [p_l^B,p_u^B]$; conditional on the event $E_B=\{p^{\ast}(B)\in[p_l^B,p_u^B]\}$, we have
    $|p^{\ast}(B)-p_t|\le p_u^B-p_l^B\le \Delta_k$ for $l(B)=k$.
    On the other hand, by Assumption~\ref{asp:maximizer} part two and three, $|p^{\ast}(\bm X_t)-p^{\ast}(B)|\le \sup\{p^{\ast}(\bm x):\bm x\in B\}-\inf\{p^{\ast}(\bm x):\bm x\in B\}\le M_4 d_B\le M_4\sqrt{d}2^{-k} $ for $l(B)=k$.
    Therefore, term 2 can be bounded by
    \begin{align}
        &\E\left[\sum_{t=1}^T\sum_{k=0}^K\sum_{\{B:l(B)=k\}}(f^{\ast}(\bm X_t)-f(\bm X_t,p_t))\I{\bm X_t\in B,B\in\mathcal P_t,E_B}\right]\nonumber\\
        &\le \E\left[\sum_{t=1}^T\sum_{k=0}^K\sum_{\{B:l(B)=k\}}M_3(\Delta_k+M_4\sqrt{d}2^{-k})^2\I{\bm X_t\in B,B\in\mathcal P_t,E_B}\right]\nonumber\\
        &\le \E\left[\sum_{k=0}^{K-1}\sum_{\{B:l(B)=k\}}\sum_{t=1}^TM_3(\Delta_k+M_4\sqrt{d}2^{-k})^2\I{\bm X_t\in B,B\in\mathcal P_t}\right]\label{eq:bin-sum}\\
        &\quad +\sum_{t=1}^TM_3(\Delta_K+M_4\sqrt{d}2^{-K})^2\sum_{\{B:l(B)=K\}}\PR(\bm X_t\in B)\nonumber
    \end{align}
    For the first term in \eqref{eq:bin-sum}, note that $\{X_t\in B,B\in \mathcal P_t\}$ occurs for at most $n_k$ times for given $B$ with $l(B)=k$. Moreover, there are $2^{dk}$ bins with level $k$, i.e., $\#\{B:l(B)=k\}=2^{dk}$.
    Therefore, substituting $\Delta_k = 2^{-k}\log(T)$ into the first term yields an upper bound $\sum_{k=0}^{K-1} M_3n_k2^{(d-2)k}(\log(T)+M_4\sqrt{d})^2$.
    For the second term in \eqref{eq:bin-sum}, $\sum_{\{B:l(B)=K\}}\PR(\bm X_t\in B)=1$ because $\{B:l(B)=K\}$ form a partition of the covariate space and $\bm X$ always falls into one of the bins.
    Therefore, \eqref{eq:bin-sum} is bounded by
    \begin{align}
        &M_3\left(\sum_{k=0}^{K-1} n_k2^{(d-2)k}+ T2^{-2K}\right)\left(\log(T)+M_4\sqrt{d}\right)^2\log(T)^2\notag\\
        \le &c_3 \log(T)^2M_3\left(\sum_{k=0}^{K-1} n_k2^{(d-2)k}+ T2^{-2K}\right)\label{eq:term-2}
    \end{align}
    for $c_3=(1+M_3)\left(\log(2)+M_4\sqrt{d}\right)^2/\log(2)^2$ and $T\ge 2$.

    Combining \eqref{eq:term-1} and \eqref{eq:term-2}, we can find constants $c_2=c_1/2^5=M^2_2/2^{15}$ such that
    \begin{align}
        \E[R_{\pi_{ABE}}]&\le \sum_{k=0}^{K-1} \left(c_3\log(T)^2n_k2^{(d-2)k}+M_1 T\exp(- c_1 2^{-4k-4}\log^3(T) n_{k}/2)\right)+c_3\log(T)^2T2^{-2K}\notag\\
        &\le \sum_{k=0}^{K-1} \left(c_3\log(T)^2n_k2^{(d-2)k}+M_1 T\exp(- c_2 2^{-4k}\log^3(T) n_{k})\right)+c_3\log(T)^2T2^{-2K} \label{eq:min-lk}
    \end{align}
    We choose $n_k$
    \begin{equation*}
        n_k = \max\left\{0, \left\lceil \frac{ 2^{4k+15}}{M^2_2\log(T)^3}(\log(T)+\log(\log(T))-(d+2)k\log(2))\right\rceil\right\}
    \end{equation*}
    to minimize $c_3\log(T)^2n_k2^{(d-2)k}+M_1 T\exp(- c_2 2^{-4k}\log(T)^3 n_{k})$ in \eqref{eq:min-lk}.
    More precisely,
    \begin{align*}
        c_3\log(T)^2n_k2^{(d-2)k}&\le c_3M_2^{-2}2^{15}\frac{\log(T)^2}{\log(T)^3}2^{(d+2)k}(\log(T)+\log(\log(T)))\\
                                    &\le c_42^{(d+2)k}
    \end{align*}
    for some constants $c_4>0$, and
    \begin{align*}
        M_1 T\exp(- c_2 2^{-4k}\log^3(T) n_{k})&\le M_1 T \exp(-\log(T)-\log(\log(T))+(d+2)k\log(2))\\
                                               &\le c_5 2^{(d+2)k}
    \end{align*}
    for a constant $c_5>0$.
    Therefore, \eqref{eq:min-lk} implies that we can find a constant $c_6=c_4+c_5$ such that
    \begin{align*}
        \E[R_{\pi_{ABE}}]&\le \sum_{k=0}^{K-1} c_6 2^{(d+2)k}+c_3\log(T)^2T2^{-2K}\\
        &\le c_6 2^{(d+2)K}+c_3\log(T)^2T2^{-2K}.
    \end{align*}
    Therefore, by our choice of $K=\lfloor \frac{\log(T)}{(d+4)\log (2)}\rfloor$, the regret is bounded by
    \begin{align*}
        c_7 \log^2(T)T^{ \frac{d+2}{d+4}}.
    \end{align*}
    for some constant $c_7$. Hence we have completed the proof.
\end{proof}

\subsection{Lower Bound in Section~\ref{sec:lower-bound}}
Next we show that Assumption 1 and 2 are satisfied by the construction in Section~\ref{sec:lower-bound}.
\begin{proposition}\label{prop:asp-satisfied}
    The choice of $f\in\mathcal C$ satisfies Assumption~\ref{asp:lipschitz} and~\ref{asp:maximizer} with $M_1=4$, $M_2=1$, $M_3=2$, and $M_4=1$.
\end{proposition}

%The proof of Proposition~\ref{prop:asp-satisfied} is provided in the online appendix.
To give some intuitions, note that by the construction of $f$, both $f_w(\bm x,p)$ and $p^{\ast}(\bm x)$ are Lipschitz continuous in $[0,1)^d$.
Such continuity guarantees the desired properties.

\begin{proof}[Proof of Proposition~\ref{prop:asp-satisfied}:]
    For Assumption~\ref{asp:lipschitz}, we discuss two cases. The first case is $\bm x_1,\bm x_2\in B_j$, i.e., the two customers are in the same bin.
    In this case,
    \begin{align*}
        |f_w(\bm x_1,p_1)-f_w(\bm x_2,p_2)|\le \begin{cases}
            |p_1(\frac{2}{3}- \frac{p_1}{2})-p_2( \frac{2}{3}- \frac{p_2}{2})|\le 2|p_1-p_2| & w_j=0\\
            2|p_1-p_2|+ |p_1( \frac{1}{3}- \frac{p_1}{2})D(\bm x_1,\partial B_j)-p_2 ( \frac{1}{3}- \frac{p_2}{2})D(\bm x_2,\partial B_j)|& w_j=1
        \end{cases}
    \end{align*}
    When $w_j=0$, the assumption is already satisfied.
    When $w_j=1$, by the triangle inequality we have
    \begin{align*}
        &|p_1( \frac{1}{3}- \frac{p_1}{2})D(\bm x_1,\partial B_j)-p_2 ( \frac{1}{3}- \frac{p_2}{2})D(\bm x_2,\partial B_j)|\\
        \le &2|p_1-p_2|D(\bm x_1,\partial B_j)+p_2\left( \frac{1}{3}- \frac{p_2}{2}\right)|D(\bm x_1,\partial B_j)-D(\bm x_2,\partial B_j)|\\
                                                               \le &\frac{1}{M}|p_1-p_2|+|D(\bm x_1,\partial B_j)-D(\bm x_2,\partial B_j)|\\
                                                               \le &|p_1-p_2|+\|\bm x_1-\bm x_2\|_2
    \end{align*}
    The second inequality is because $p_2\ge 0$ and $D(\bm x_1,\partial B_j)\le 1/2M\le 1/2$ when $\bm x_1\in B_j$.
    The third inequality is because
    \begin{align*}
        \|\bm x_1-\bm x_2\|_2+D(\bm x_2,\partial B_j)&=\min_{a\in \partial B_j}\{\|a-\bm x_2\|_2+\|\bm x_1-\bm x_2\|_2\}\ge \min_{\bm a\in \partial B_j}\{\|a-\bm x_1\|_2\}\\
        &=D(\bm x_1,\partial B_j)
    \end{align*}
    and similarly $\|\bm x_1-\bm x_2\|_2+D(\bm x_1,\partial B_j)\ge D(\bm x_2,\partial B_j)$.
    Therefore, we have shown that $|f_w(\bm x_1,p_1)-f_w(\bm x_2,p_2)|\le 3|p_1-p_2|+2\|\bm x_1-\bm x_2\|_2$ for case one.

    The second case is $\bm x_1\in B_{j_1}$ and $\bm x_2\in B_{j_2}$ for $j_1\neq j_2$.
    If $j_1=j_2=0$, then by the previous analysis, we already have $|f_w(\bm x_1,p_1)-f_w(\bm x_2,p_2)|\le 2|p_1-p_2|$.
    If $j_1=1$ and $j_2=0$, then
    \begin{align*}
        |f_w(\bm x_1,p_1)-f_w(\bm x_2,p_2)|\le 2|p_1-p_2|+p_1\left( \frac{1}{3}- \frac{p_1}{2}\right)D(\bm x_1,\partial B_{j_1})\le2|p_1-p_2|+2\|\bm x_1-\bm x_2\|_2.
    \end{align*}
    The last inequality is because the straight line connecting $\bm x_1$ and $\bm x_2$ must intersects $B_{j_1}$. The distance from $\bm x_1$ to the intersection is no less than $D(\bm x_1,\partial B_{j_1})$. Therefore, $D(\bm x_1,\partial B_{j_1})\le \|\bm x_1-\bm x_2\|_2$.
    If $j_1=0$ and $j_2=1$, then the result follows similarly.
    If $j_1=j_2=1$, then by the same argument,
    \begin{align*}
        |p_1\left( \frac{1}{3}- \frac{p_1}{2}\right)D(\bm x_1,\partial B_{j_1})-p_2\left(  \frac{1}{3}- \frac{p_2}{2}\right)D(\bm x_2,\partial B_{j_2})|\le D(\bm x_1,\partial B_{j_1})+D(\bm x_2,\partial B_{j_2})\le 2\|\bm x_1-\bm x_2\|_2.
    \end{align*}
    Therefore, combining both cases, we always have
    \begin{align*}
        |f_w(\bm x_1,p_1)-f_w(\bm x_2,p_2)|\le 4|p_1-p_2|+4\|\bm x_1-\bm x_2\|_2.
    \end{align*}

    For Assumption~\ref{asp:maximizer}, note that
    \begin{align*}
        f_B(p)&=p\left( \frac{2}{3}- \frac{p}{2}\right)+p\left( \frac{1}{3}- \frac{p}{2}\right)\E\left[\sum_{j=1}^{M^d} w_j \I{\bm X\in B_j}D(\bm X,\partial B_j)|\bm X\in B\right]\\
              &= p\left( \frac{2}{3}- \frac{p}{2}\right)+p\left( \frac{1}{3}- \frac{p}{2}\right)\sum_{j=1}^{M^d}w_j\PR(B_j\cap B)\E[D(\bm X,\partial B_j)|\bm X\in B_j\cap B].
    \end{align*}
    For part one, because the second-order derivative of $f_B(p)$ is always bounded between $[-2,-1]$,
    \begin{align*}
        2(p-p^{\ast}(B))^2\ge f_B(p^{\ast}(B))-f_B(p)\ge (p-p^{\ast}(B))^2
    \end{align*}
    and part one holds for $M_2=1$, $M_3=2$.
    For part two, note that the maximizer
    \begin{align*}
    p^{\ast}(B)= \frac{2+\sum_{j=1}^{M^d}w_j\PR(B_j\cap B)\E[D(\bm X,\partial B_j)|\bm X\in B_j\cap B]}{3(1+\sum_{j=1}^{M^d}w_j\PR(B_j\cap B)\E[D(\bm X,\partial B_j)|\bm X\in B_j\cap B])}.
    \end{align*}
    Because $p^{\ast}(B)$ is a monotone function of $\sum_{j=1}^{M^d}w_j\PR(B_j\cap B)\E[D(\bm X,\partial B_j)|\bm X\in B_j\cap B]$,
%    If we regard $p^{\ast}(\bm X')$ as a function of the random variable $\bm X'$ that is uniformly distributed in $B$,
%    then $p^{\ast}(B)=\E[p^{\ast}(\bm X')]$.
    it is easy to check that part two of Assumption~\ref{asp:maximizer} holds.
    For part three, consider $\bm x_1\in B\cap B_{j_1}$ and $\bm x_2\in B\cap B_{j_2}$.
    If $w_{j_1}=0$ and $w_{j_2}=0$, then $p^{\ast}(\bm x_1)- p^{\ast}(\bm x_2)=0$.
    If either $w_{j_1}=0$ or $w_{j_2}=0$, then
    \begin{equation*}
        |p^{\ast}(\bm x_1)- p^{\ast}(\bm x_2)|  \le \frac{1}{3}\max\left\{D(\bm x_1,\partial B_{j_1}),D(\bm x_2,\partial B_{j_2})\right\}\le \|\bm x_1-\bm x_2\|_2\le d_B.
    \end{equation*}
    by the previous analysis.
    If $w_{j_1}=1$ and $w_{j_2}=1$, then similarly we have $|p^{\ast}(\bm x_1)- p^{\ast}(\bm x_2)|\le \frac{1}{3}|D(\bm x_1,\partial B_{j_1})-D(\bm x_2,\partial B_{j_2})|\le \|\bm x_1-\bm x_2\|_2\le d_B$.
    Therefore, part three holds with $M_4=1$.
\end{proof}

The proof of Theorem~\ref{thm:lower-bound} uses Kullback-Leibler (KL) divergence to measure the ``distinguishability'' of the underlying functions.
Such information-theoretic approach has been a standard technique in the learning literature.
The proof is outlined in the following.

Among all functions $f\in \mathcal C$, we focus on each pair of $f_w$ and $f_{w'}$ that only differ in a single bin.
For example, consider $w=(w_1,w_2,\ldots,w_{j-1},0,w_{j+1},\ldots,w_{M^d})$ and $w'=(w_1,w_2,\ldots,w_{j-1},1,w_{j+1},\ldots,w_{M^d})$ for some $j$.
Because the indices of $w$ and $w'$ are identical except for the $j$th, $f_w$ and $f_{w'}$ only differ in bin $B_j$.
Denote $w=(w_{-j},0)$ and $w'=(w_{-j},1)$ to highlight this fact.
Distinguishing between $f_{w_{-j},0}$ and $f_{w_{-j},1}$ poses a challenge to any policy.
%relax the regret $\sup_{f\in\mathcal C}\E[R_\pi]\le\frac{1}{2^{M^d}}\sum_{f\in\mathcal C}\E[R_\pi]$ and focus on the average regret of all possible underlying functions,
%where we note that the expectation depends on the underlying function $f(\bm x, p)$ as well as the policy $\pi$ that dictates how rewards are generated and affect future decisions.
%Such relaxation allows us to focus on a pair of $w$ and ${w'}$ that only differ in a bin $B_j$, i.e., only the $j$th index of $w$ and $w'$ are different.
In particular, for $\bm x\in B_j$, the difference of the two functions $|f_{w_{-j},0}(\bm x, p)-f_{w_{-j},1}(\bm x,p)|=p( \frac{1}{3}- \frac{p}{2})D(\bm x,\partial B_j)$ is diminishing when in $p\approx 2/3$.
Thus, charging a price different from $2/3$ makes the difference more visible and helps to distinguish $f_{w_{-j},0}$ and $f_{w_{-j},1}$.
However, if $p$ deviates too much from the optimal decision $p^{\ast}(\bm x)=2/3$ or $p^{\ast}(\bm x)=(2+D(\bm x,\partial B_j))/(3(1+D(\bm x,\partial B_j)))$,
then significant regret is incurred in that period.

To capture this trade-off, for a given $j=1,\dots,M^d$ and $w_{-j}\in \{0,1\}^{M^d-1}$, define the following quantity
\begin{equation}\label{eq:def-z}
    z_{w_{-j}} =\sum_{t=1}^T\frac{9}{11M^2}\E^{\pi}_{f_{w_{-j},0}}\left[\left( \frac{2}{3}-p_t\right)^2\I{\bm X_t\in B_j}\right]
\end{equation}
where the expectation is taken with respect to a policy $\pi$ and the underlying function $f_{w_{-j},0}$.
This quantity is crucial in analyzing the regret.
More precisely, if $z_{w_{-j}}$ is large (which implies that $p_t$ is large), then $f_{w_{-j},0}$ and $f_{w_{-j},1}$ are easy to distinguish but the regret becomes uncontrollable.
\begin{lemma}\label{lem:z-big}
    \begin{equation*}
       \sup_{f\in \mathcal C} R_{\pi}\ge \frac{11M_2}{9\times2^{M^d}}M^2 \sum_{j=1}^{M^d}\sum_{w_{-j}} z_{w_{-j}}.
    \end{equation*}
\end{lemma}
On the other hand, if $z_{w_{-j}}$ is small, then the KL divergence of the measures associated with $f_{w_{-j},0}$ and $f_{w_{-j},1}$ is also small.
In other words, the firm cannot easily distinguish between $f_{w_{-j},0}$ and $f_{w_{-j},1}$ which impedes learning and incurs substantial regret.
\begin{lemma}\label{lem:z-small}
    \begin{equation*}
        \sup_{f\in \mathcal C}R_{\pi}\ge \frac{M_2T}{9\times 2^{M^d+11}M^{d+2}} \sum_{j=1}^{M^d}\sum_{w_{-j}} \exp(-z_{w_{-j}}).
    \end{equation*}
\end{lemma}
Since the effects of $z_{w_{-j}}$ are opposite in Lemma~\ref{lem:z-big} and Lemma~\ref{lem:z-small}, combining the two bounds, we can find a positive constant $c_1$ independent of $T$ and $M$ so that
\begin{align*}
    R_{\pi}& \ge \frac{c_1}{2^{M^d}}\sum_{j=1}^{M^d}\sum_{w_{-j}}\left( \frac{T }{M^{d+2}}\exp\left(-z_{w_{-j}}\right)+ M^2 z_{w_{-j}}\right)\\
                                                                                       &\ge \frac{c_1}{2^{M^d}}\sum_{j=1}^{M^d}\sum_{w_{-j}}M^2\left(1+\log\left( \frac{T }{M^{d+4}}\right)\right)\\
                                                                                       &\ge \frac{c_1M^{d+2}}{2}\left(1+\log\left( \frac{T }{M^{d+4}}\right)\right).
\end{align*}
In the second inequality above, we minimize the expression over positive ${z_{w_{-j}}}$.
Since $M$ can be an arbitrary positive integer, we let $M=\lceil T^{1/(d+4)}\rceil$ in the last quantity.
Calculation shows that it is lower bounded by $c T^{(2+d)/(4+d)}$ for a constant $c>0$.

To prove Theorem~\ref{thm:lower-bound}, i.e., Lemma~\ref{lem:z-big} and Lemma~\ref{lem:z-small}, we introduce the following lemmas.

\begin{lemma}[KL divergence for Bernoulli Random Variables]\label{lem:normal-kl}
    For two Bernoulli random variables $X_1$ and $X_2$ with means $\theta_1$ and $\theta_2$, we have
    \begin{equation*}
        \mathcal K(\mu_{X_1},\mu_{X_2})\le \frac{(\theta_1-\theta_2)^2}{\theta_2(1-\theta_2)}.
    \end{equation*}
\end{lemma}
\begin{proof}
    The proof can be found in, e.g., \citet{rigollet2010nonparametric} and thus omitted.
\end{proof}

\begin{lemma}[The chain rule of the KL divergence]\label{lem:kl-chain}
    Given joint distributions $p(x,y)$ and $q(x,y)$, we have
    \begin{equation*}
        \mathcal K(p(x,y),q(x,y))=\mathcal K(p(x),q(x))+ \E_{p(x)}[\mathcal K(p(y|x),q(y|x)],
    \end{equation*}
    where $p(\cdot)$ and $q(\cdot)$ represent the marginal distribution, $p(\cdot|x)$ and $q(\cdot|x)$ represent the conditional distribution.
\end{lemma}
\begin{proof}
    The proof can be found from standard textbooks and is thus omitted.
\end{proof}

\begin{proof}[Proof of Lemma~\ref{lem:z-big}:]
We use $\E_f^{\pi}$ to highlight the dependence of the expectation on the policy $\pi$ and the underlying function $f$.
Note that
\begin{align*}
    \sup_{f\in\mathcal C}R_\pi=\sup_{f\in\mathcal C} \sum_{t=1}^T\E\left[f^{\ast}(\bm X_t)-f(\bm X_t,p_t)\right]&=\sup_{f\in\mathcal C} \sum_{t=1}^T\sum_{j=1}^{M^d}\E\left[(f^{\ast}(\bm X_t)-f(\bm X_t,p_t))\I{\bm X_t\in B_j}\right]\\
                                                                                       &\ge M_2\sup_{f\in \mathcal C}\sum_{t=1}^T\sum_{j=1}^{M^d}\E\left[(p^{\ast}(\bm X_t)-p_t)^2\I{\bm X_t\in B_j}\right]\\
                                                                                       &\ge \frac{M_2}{2^{M^d}}\sum_{w}\sum_{t=1}^T\sum_{j=1}^{M^d}\E^\pi_{f_w}\left[(p^{\ast}(\bm X_t)-p_t)^2\I{\bm X_t\in B_j}\right].
\end{align*}
In the last inequality, we have used the fact that $\#\{\mathcal C\}=2^{M^d}$ and the supremum is always no less than the average.

For a given bin $B_j$, we focus on $f_{w_{-j},0}$ and $f_{w_{-j},1}$, which only differ for  $\bm x\in B_j$.
Therefore, we can rearrange $\sum_{w}$ to $\sum_{w_{-j}\in \{0,1\}^{M^d-1}}\sum_{w_j\in\{0,1\}}$.
We have the following lower bound for the regret
\begin{align*}
    \sup_{f\in\mathcal C} R_{\pi}
    &\ge \frac{M_2}{2^{M^d}}\sum_{j=1}^{M^d}\sum_{w_{-j}\in\left\{0,1\right\}^{M^d-1}}\sum_{w_j\in\{0,1\}}\sum_{t=1}^T\E^\pi_{f_{w_{-j},w_j}}\left[(p^{\ast}(\bm X_t)-p_t)^2\I{\bm X_t\in B_j}\right]\\
    &\ge \frac{M_2}{2^{M^d}}\sum_{j=1}^{M^d}\sum_{w_{-j}}\sum_{t=1}^T\E^\pi_{f_{w_{-j},0}}\left[(p^{\ast}(\bm X_t)-p_t)^2\I{\bm X_t\in B_j}\right]\\
    &= \frac{M_2}{2^{M^d}}\sum_{j=1}^{M^d}\sum_{w_{-j}}\sum_{t=1}^T\E^\pi_{f_{w_{-j},0}}\left[\left( \frac{2}{3}-p_t\right)^2\I{\bm X_t\in B_j}\right]\\
    &= \frac{11M_2}{9\times 2^{M^d}}M^2\sum_{j=1}^{M^d}\sum_{w_{-j}}z_{w_{-j}}.
\end{align*}
In the second inequality, we have neglected the regret for $f_{w_{-j},1}$.
The last equality is by the definition of $z_{w_{-j}}$ in \eqref{eq:def-z}.
Hence we have proved the result.
\end{proof}

\begin{proof}[Proof of Lemma~\ref{lem:z-small}:]
    By the same argument as in the proof of Lemma~\ref{lem:z-big}, we have
    \begin{equation*}
        \sup_{f\in\mathcal C} R_{\pi}
    \ge \frac{M_2}{2^{M^d}}\sum_{j=1}^{M^d}\sum_{t=1}^T\sum_{w_{-j}\in\left\{0,1\right\}^{M^d-1}}\sum_{w_j\in\{0,1\}}\E^\pi_{f_{w_{-j},w_j}}\left[(p^{\ast}(\bm X_t)-p_t)^2\I{\bm X_t\in B_j}\right].
    \end{equation*}
    Because $\bm X_t$ is uniformly distributed in $[0,1)^d$, $\PR(\bm X_t\in B_{j})=M^{-d}$.
    By conditioning on the event $\bm X_t\in B_j$, we have
    \begin{align*}
        \E^{\pi}_{f_{w_{-j},w_j}}\left[(p^{\ast}(\bm X_t)-p_t)^2\I{\bm X_t\in B_j}\right]&=\E^{\pi}_{f_{w_{-j},w_j}}\left[(p^{\ast}(\bm X_t)-p_t)^2|\bm X_t\in B_j\right]\PR(\bm X_t\in B_j)\\
                                                                                           &= \frac{1}{M^d}\E^{\pi}_{f_{w_{-j},w_j}}\left[(p^{\ast}(\bm X_t)-p_t)^2|\bm X_t\in B_j\right].
    \end{align*}
    Since $(p^{\ast}(\bm X_t)-p_t)^2$ is measurable with respect to the $\sigma$-algebra generated by $\mathcal F_{t-1}$ and $\bm X_t$, by the tower property, we have
    \begin{align*}
        \E^{\pi}_{f_{w_{-j},w_j}}\left[(p^{\ast}(\bm X_t)-p_t)^2\I{\bm X_t\in B_j}\right]&= \frac{1}{M^d}\E^{\pi}_{f_{w_{-j},w_j}}\left[\E\left[(p^{\ast}(\bm X_t)-p_t)^2|\mathcal F_{t-1},\bm X_t\in B_j\right]\right]
    \end{align*}
    Let $\E^{\pi,t-1}_{f_{w_{-j},w_j}}[\cdot]$ denote $\E^{\pi}_{f_{w_{-j},w_j}}\left[\E[\cdot |\mathcal F_{t-1}]\right]$ and let $\PR_{\bm X_t}^{B,t-1}(\cdot)$ denote the conditional probability $\PR(\cdot|\mathcal F_{t-1},\bm X_t\in B)$.
    By Markov's inequality, for any constant $s>0$ we have
    \begin{align}
        &\sum_{w_j\in\{0,1\}} \E^\pi_{f_{w_{-j},w_j}}\left[(p^{\ast}(\bm X_t)-p_t)^2\I{\bm X_t\in B_j}\right]\label{eq:lower-bound-decomp}\\
        &= \frac{1}{M^d}\sum_{w_j\in\left\{0,1\right\}} \E^{\pi}_{f_{w_{-j},w_j}}\left[\E\left[(p^{\ast}(\bm X_t)-p_t)^2|\bm X_t\in B_j,\mathcal F_{t-1}\right]\right]\notag\\
        &\ge \frac{1}{M^d}\sum_{w_j\in\left\{0,1\right\}} \frac{s^2}{M^2} \E^{\pi}_{f_{w_{-j},w_j}}\left[\PR_{\bm X_t}^{B_j,t-1}\left(\big|p^{\ast}(\bm X_t)-p_t\big|\ge \frac{s}{M}\right)\right]\notag\\
        &= \frac{s^2}{M^{d+2}} \left(\E^{\pi}_{f_{w_{-j},0}}\left[ \PR_{\bm X_t}^{B_j,t-1}\left(| \frac{2}{3}-p_t|\ge \frac{s}{M}\right)\right]+\E^{\pi}_{f_{w_{-j},1}}\left[ \PR_{\bm X_t}^{B_j,t-1}\left(|
        \frac{2+D(\bm X_t,\partial B_j)}{3(1+D(\bm X_t,\partial B_j))}-p_t|\ge \frac{s}{M}\right)\right]\right)\notag\\
        &\ge \frac{s^2}{M^{d+2}} \bigg(\E^{\pi}_{f_{w_{-j},0}}\left[ \PR_{\bm X_t}^{B_j,t-1}\left(| \frac{2}{3}-p_t|\ge \frac{s}{M},A\right)\right]\\
        &\quad +\E^{\pi}_{f_{w_{-j},1}}\left[ \PR_{\bm X_t}^{B_j,t-1}\left(| \frac{2+D(\bm X_t,\partial B_j)}{3(1+D(\bm X_t,\partial B_j))}-p_t|\ge \frac{s}{M},A\right)\right]\bigg)\notag
    \end{align}
    where we define event $A=\{\bm X_t\in B_j\}\cap\{D(\bm X_t,\partial B_j)>12s/M\}$.
    In the second equality, we have used the fact that for $\bm X_t\in B_j$, when $w_j=0$, $p^{\ast}(\bm X_t)= \frac{2}{3}$; when $w_j=1$, $p^{\ast}(\bm X_t)=\frac{2+D(\bm X_t,\partial B_j)}{3(1+D(\bm X_t,\partial B_j))}$.
    The motivation of introducing $A$ is as follows: consider the classification rule $\Pi_t\mapsto \{0,1\}$ associated with $p_t$ tries to distinguish between $w_j=0$ and $w_j=1$.
    It is defined as
    \begin{equation*}
        \Pi_t = \begin{cases}
            0 & | \frac{2}{3}-p_t|\le |\frac{2+D(\bm X_t,\partial B_j)}{3(1+D(\bm X_t,\partial B_j))}-p_t|\\
            1 & \text{otherwise}.
        \end{cases}
    \end{equation*}
    In other words, $\Pi_t$ classifies the underlying function as $f_{w_{-j},0}$ if $p_t$ is closer to the optimal price $p^{\ast}(\bm X_t)$ of $f_{w_{-j},0}$, and as $f_{w_{-j},1}$ vice versa.
    For $f_{w_{-j},0}$, a misclassification on the event $A$ is $A\cap\{\Pi_t=1\}$.
    It implies that
    \begin{align*}
        |2/3-p_t| & \ge \frac{1}{2}\times \left| \frac{2}{3}- \frac{2+D(\bm X_t,\partial B_j)}{3(1+D(\bm X_t,\partial B_j))}\right| = \frac{D(\bm X_t,\partial B_j)}{6(1+D(\bm X_t,\partial B_j))},
    \end{align*}
    which on $A$, implies $A\cap \{|2/3-p_t|\ge s/M\}$ as $D(\bm X_t,\partial B_j)\le 1$.
    Similarly, $A\cap\{\Pi_t=0\}\subset A\cap \{|\frac{2+D(\bm X_t,\partial B_j)}{3(1+D(\bm X_t,\partial B_j))}-p_t|\ge s/M\}$.
    Therefore, by the fact that $\PR(\bm X\in A)=(1-24s)^d/M^d$, we have
    \begin{align}
        &\E^{\pi}_{f_{w_{-j},0}}\left[ \PR_{\bm X_t}^{B_j,t-1}\left(|1-p_t|\ge \frac{s}{M},A\right)\right]+\E^{\pi}_{f_{w_{-j},1}}\left[ \PR_{\bm X_t}^{B_j,t-1}\left(|1-D(\bm X_t,\partial B_j)-p_t|\ge \frac{s}{M},A\right)\right]\notag\\
        &\ge\E^{\pi}_{f_{w_{-j},0}}\left[ \PR_{\bm X_t}^{B_j,t-1}\left(A\cap \{\Pi_t=1\}\right)\right]+\E^{\pi}_{f_{w_{-j},1}}\left[ \PR_{\bm X_t}^{B_j,t-1}\left(A\cap\{\Pi_t=0\}\right)\right]\notag\\
        &= (1-24s)^d\left( \PR_{f_{w_{-j},0}}^{\pi}\left(\Pi_t=1|\bm X_t\in A\right)+\PR_{f_{w_{-j},1}}^{\pi}\left(\Pi_t=0|\bm X_t\in A\right)\right)\label{eq:mis-error}.
    \end{align}
    Next we lower bound the misclassification error \eqref{eq:mis-error} by the
    Kullback-Leibler (KL) divergence between the two probability measures associated with $f_{w_{-j},0}$ and $f_{w_{-j},1}$.
    Intuitively, if the two probability measures are close, then no classification (including $\Pi_t$) can incur very small misclassification error.
    Formally, introduce the KL divergence between two probability measures $P$ and $Q$ as
    \begin{equation*}
        \mathcal K(P,Q)= \begin{cases}
            \int \log \frac{dP}{dQ}dP& \text{if }P\ll Q\\
            +\infty &\text{otherwise}
        \end{cases},
    \end{equation*}
    where $P\ll Q$ indicates that $P$ is absolute continuous w.r.t. $Q$.
    By the independence of $\mathcal F_{t-1}$ and $\bm X_t$, the two measures we want to distinguish in \eqref{eq:mis-error}, $\mu_{f_{w_{-j},0}}^{\pi}(\cdot|\bm X_t\in A)$ and $\mu_{f_{w_{-j},1}}^{\pi}(\cdot|\bm X_t\in A)$, can be expressed as product measures
    \begin{align*}
        \mu_{f_{w_{-j},0}}^{\pi}(\cdot|\bm X_t\in A)&=\mu_{f_{w_{-j},0}}^{\pi,t-1}(\cdot)\times \mu_{\bm X_t}^{A}(\cdot)\\
        \mu_{f_{w_{-j},1}}^{\pi}(\cdot|\bm X_t\in A)&=\mu_{f_{w_{-j},1}}^{\pi,t-1}(\cdot)\times \mu_{\bm X_t}^{A}(\cdot),
    \end{align*}
    where $\mu_{f_{w_{-j},0}}^{\pi,t-1}(\cdot)$ is a measure of $(\bm X_1,Z_1,\dots,\bm X_{t-1},Z_{t-1})$ depending on $\pi$ and $f_{w_{-j},0}$ and $\mu_{\bm X_t}^{A}(\cdot)$ is a measure of $\bm X_t$ conditional on $\bm X_t\in A$.
    By Theorem 2.2 (iii) in \citet{tsybakov2009introduction},
    \begin{align}
        \eqref{eq:mis-error}&\ge \frac{(1-24s)^d}{2}\exp\left(-\mathcal K\left(\mu^{\pi,t-1}_{f_{w_{-j},0}}\times \mu_{\bm X_t}^{A},\mu^{\pi,t-1}_{f_{w_{-j},1}}\times \mu_{\bm X_t}^{A}\right)\right)\notag\\
                            &= \frac{(1-24s)^d}{2}\exp\left(-\mathcal K\left(\mu^{\pi,t-1}_{f_{w_{-j},0}},\mu^{\pi,t-1}_{f_{w_{-j},1}}\right)-\E^{\pi,t-1}_{f_{w_{-j},0}}\left[\mathcal K\left(\mu_{\bm X_t}^{A},\mu_{\bm X_t}^{A}\right)\right]\right)\notag\\
                            &= \frac{(1-24s)^d}{2}\exp\left(-\mathcal K\left(\mu^{\pi,t-1}_{f_{w_{-j},0}},\mu^{\pi,t-1}_{f_{w_{-j},1}}\right)\right).\label{eq:lb-to-kl}
    \end{align}
    The second line follows from Lemma~\ref{lem:kl-chain}; the third line follows from the fact that
    $\mu_{\bm X_t}^A$ is the same distribution for $f_{w_{-j},0}$ and $f_{w_{-j},1}$, independent of $\mathcal F_{t-1}$.

    To further simplify the expression, note that $\mu^{\pi,t}_{f_{w_{-j},0}}(\cdot)$ can be decomposed as
    \begin{equation*}
        \mu^{\pi,t}_{f_{w_{-j},0}}(\cdot)=\mu^{\pi,t-1}_{f_{w_{-j},0}}(\cdot)\times\mu_{\bm X}(\cdot)\times \mu_{f_{w_{-j},0}}^{Z_t}(\cdot|\mathcal F_{t-1},\bm X_t),
    \end{equation*}
    where $\mu_{\bm X}$ is the measure (uniform distribution) of $\bm X_t$ and $\mu_{f_{w_{-j},0}}^{Z_t}(\cdot|\mathcal F_{t-1},\bm X_t)$ is the measure of $Z_t$ conditional on $\mathcal F_{t-1}$ and $\bm X_t$.
    We apply Lemma~\ref{lem:kl-chain} again:
    \begin{align*}
        \mathcal K\left(\mu^{\pi,t}_{f_{w_{-j},0}},\mu^{\pi,t}_{f_{w_{-j},1}}\right)&=\mathcal K\left(\mu^{\pi,t-1}_{f_{w_{-j},0}},\mu^{\pi,t-1}_{f_{w_{-j},1}}\right)+\E^{\pi,t-1}_{f_{w_{-j},0}}\left[\mathcal K(\mu_{\bm X_t},\mu_{\bm X_t})\right]\\
                                                                                    &\quad +\E^{\pi,t-1}_{f_{w_{-j},0}}\left[\E_{\bm X}\left[\mathcal K\left(\mu^{Z_t}_{f{w_{-j},0}}(\cdot|\mathcal F_{t-1},\bm X_t),\mu^{Z_t}_{f{w_{-j},1}}(\cdot|\mathcal F_{t-1},\bm X_t)\right)\right]\right].
    \end{align*}
    It is easy to see that the second term is zero.
    For the third term, we first conditional on $\mathcal F_{t-1}$ and then on the covariate $\bm X_t$.
    Because $p_t$ depends only on $\mathcal F_{t-1}$ and $\bm X_{t}$, $p_t$ is the same for $f_{w_{-j},0}$ and $f_{w_{-j},1}$ conditional on $\mathcal F_{t-1}$ and $\bm X_{t}$.
    Therefore, $\mu^{Z_t}_{f{w_{-j},0}}(\cdot|\mathcal F_{t-1},\bm X_t)$ and $\mu^{Z_t}_{f{w_{-j},1}}(\cdot|\mathcal F_{t-1},\bm X_t)$
    are two Bernoulli distributions with means $d_{w_{-j},0}(\bm X_t,p_t)$ and $d_{w_{-j},1}(\bm X_t,p_t)$, respectively.
    By Lemma~\ref{lem:normal-kl}, we have
    \begin{align*}
        \mathcal K\left(\mu^{Z_t}_{f{w_{-j},0}}(\cdot|\mathcal F_{t-1},\bm X_t),\mu^{Z_t}_{f{w_{-j},1}}(\cdot|\mathcal F_{t-1},\bm X_t)\right)&\le \frac{\left(d_{w_{-j},0}(\bm X_t,p_t)-d_{w_{-j},1}(\bm X_t,p_t)\right)^2}{d_{w_{-j},1}(\bm X_t,p_t) (1-d_{w_{-j},1}(\bm X_t,p_t))}\\
                                                                                                                                              &\le \frac{144}{11}\left(d_{w_{-j},0}(\bm X_t,p_t)-d_{w_{-j},1}(\bm X_t,p_t)\right)^2\\
                                        &= \frac{144}{11}\left(p_t\left( \frac{1}{3}- \frac{p_t}{2}\right) D(\bm X_t,\partial B_j)\I{\bm X_t\in B_j}\right)^2\\
                                        &\le  \frac{9}{11M^2}\left( \frac{2}{3}- p_t\right)^2\I{\bm X_t\in B_j}.
    \end{align*}
    In the second inequality, we used the fact that $d_{w_{-j},1}(\bm X_t,p_t)\in[ 1/12,5/6]$ as long as we choose $M\ge 2$ and thus $D(\bm X_t,\partial B_j)\le 1/2$.
    In the last inequality, we have used the fact that the distance of a vector inside $B_j$ to the boundary of $B_j$ is at most $1/2M$.
    Therefore, we can obtain an upper bound for $\mathcal K\left(\mu^{\pi,t}_{f_{w_{-j},0}},\mu^{\pi,t}_{f_{w_{-j},1}}\right)$
    \begin{align*}
        \mathcal K\left(\mu^{\pi,t}_{f_{w_{-j},0}},\mu^{\pi,t}_{f_{w_{-j},1}}\right)&\le \sum_{i=1}^t\E^{\pi,i-1}_{f_{w_{-j},0}}\left[\E_{\bm X}\left[\mathcal K\left(\mu^{Z_i}_{f{w_{-j},0}}(\cdot|\mathcal F_{i-1},\bm X_i),\mu^{Z_i}_{f{w_{-j},1}}(\cdot|\mathcal F_{i-1},\bm X_i)\right)\right]\right]\\
                                                                                    &\le \sum_{i=1}^t\E^{\pi,i-1}_{f_{w_{-j},0}}\left[\E_{\bm X}\left[ \frac{9}{11M^2}\left( \frac{2}{3}-p_t\right)^2\I{\bm X_i\in B_j}\right]\right]\\
                                                                                    &\le\sum_{t=1}^T\E^{\pi}_{f_{w_{-j},0}}\left[\frac{9}{11M^2}\left( \frac{2}{3}-p_t\right)^2\I{\bm X_t\in B_j}\right]=z_{w_{-j}}.
    \end{align*}
    Therefore, combining it with \eqref{eq:lower-bound-decomp}, \eqref{eq:mis-error} and \eqref{eq:lb-to-kl}, we have shown the lemma:
    \begin{align*}
        &\sup_{f\in\mathcal C} \sum_{t=1}^T\E\left[f^{\ast}(\bm X_t)-f(\bm X_t,p_t)\right]\\
        &\ge \frac{M_2}{2^{M^d}}\sum_{t=1}^T\sum_{j=1}^{M^d}\sum_{w_{-j}}\sum_{w_j\in\{0,1\}}\E^\pi_{f_{w_{-j},w_j}}\left[(p^{\ast}(\bm X_t)-p_t)^2\I{\bm X_t\in B_j}\right]\\
        &\ge \frac{TM_2s^2(1-24s)^2}{2^{M^d+1}M^{d+2}}\sum_{j=1}^{M^d}\sum_{w_{-j}}\exp\left(-z_{w_{-j}}\right)\\
        &= \frac{M_2T}{9\times 2^{M^d+11}M^{d+2}}\sum_{j=1}^{M^d}\sum_{w_{-j}}\exp\left(-z_{w_{-j}}\right)
    \end{align*}
    where in the last step, we have set $s=1/24$.
\end{proof}

\end{document}